\documentclass[11pt,dvipsnames]{article}
\usepackage[left=1in, top=1in, bottom=1in, right=1in]{geometry}
\usepackage{times}

\usepackage{natbib}

\usepackage{authblk}

\usepackage[colorlinks=true,citecolor=ForestGreen]{hyperref}
\usepackage{url}

\usepackage[dvipsnames]{xcolor}

\usepackage[textsize=tiny,textwidth=3cm]{todonotes}
\usepackage[normalem]{ulem}  

\usepackage{amsmath, amssymb, amsthm, bm, bbm,dsfont, amsfonts, mathalfa, mathrsfs}
\usepackage{mathtools}
\usepackage[mathscr]{eucal}
\usepackage{soul}
\usepackage{nicefrac}

\usepackage{graphicx}
\usepackage{subcaption}
\usepackage{wrapfig}

\usepackage{booktabs}
\usepackage{array}
\usepackage{colortbl}
\usepackage{multirow}
\newcolumntype{M}[1]{>{\centering\arraybackslash}m{#1}}
\newcolumntype{L}[1]{>{\raggedright\arraybackslash}m{#1}}

\usepackage{algorithm} 
\usepackage{algpseudocode}

\usepackage[utf8]{inputenc} 
\usepackage[T1]{fontenc}    
\usepackage{microtype}      
\usepackage{appendix}
\usepackage{enumitem}
\usepackage[symbol]{footmisc}  

\usepackage[capitalize,nameinlink,noabbrev]{cleveref}
\crefformat{equation}{Eq.~(#2#1#3)}
\crefmultiformat{equation}{(#2#1#3)}%
{ and~(#2#1#3)}{, (#2#1#3)}{ and~(#2#1#3)}
\crefformat{theorem}{Thm.~#2#1#3}
\crefformat{proposition}{Prop.~#2#1#3}
\crefformat{lemma}{Lemma~#2#1#3}
\crefformat{remark}{Remark~#2#1#3}
\crefformat{section}{Sec.~#2#1#3}
\crefformat{assumption}{Assumption~#2#1#3}
\crefformat{example}{Example~#2#1#3}
\crefformat{figure}{Fig.~#2#1#3}
\crefformat{algorithm}{Algorithm~#2#1#3}
\crefrangeformat{section}{Sec.~#3#1#4--#5#2#6}
\crefformat{assumption}{Assumption~#2#1#3}
\crefrangeformat{equation}{~(#3#1#4--#5#2#6)}
\crefrangeformat{figure}{Fig.~#3#1#4--#5#2#6}
\crefrangeformat{assumption}{Assumptions~(#3#1#4--#5#2#6)}
\crefformat{appendix}{Appendix~#2#1#3}

\newtheorem{theorem}{Theorem}[section]
\newtheorem{proposition}[theorem]{Proposition}
\newtheorem{lemma}[theorem]{Lemma}

\theoremstyle{definition}

\newtheorem{assumption}[theorem]{Assumption}
\theoremstyle{remark}
\newtheorem{remark}[theorem]{Remark}

\usepackage[scaled=0.85]{FiraMono}

\def\1{\bm{1}}

\def\btheta{{\bm{\theta}}}
\def\bzeta{{\bm{\zeta}}}
\def\bpsi{{\bm{\psi}}}

\newcommand{\BCal}{\mathscr{B}}

\newcommand{\DCal}{\mathscr{D}}

\newcommand{\LCal}{\mathscr{L}}

\newcommand{\SCal}{\mathscr{S}}

\newcommand{\AC}{\mathcal{A}}

\newcommand{\FC}{\mathcal{F}}

\newcommand{\MC}{\mathcal{M}}

\newcommand{\TC}{\mathcal{T}}
\newcommand{\UC}{\mathcal{U}}
\newcommand{\VC}{\mathcal{V}}

\def\re{{\textnormal{e}}}

\def\rve{{\mathbf{e}}}

\def\rvx{{\mathbf{x}}}

\def\rvz{{\mathbf{z}}}

\def\rmX{{\mathbf{X}}}

\DeclareMathAlphabet{\mathsfit}{\encodingdefault}{\sfdefault}{m}{sl}
\SetMathAlphabet{\mathsfit}{bold}{\encodingdefault}{\sfdefault}{bx}{n}

\newcommand{\R}{\mathbb{R}}

\DeclareMathOperator*{\argmax}{arg\,max}
\DeclareMathOperator*{\argmin}{arg\,min}

\newcommand{\hla}[1]{\colorbox{green!20!white}{#1}}

\newcommand{\hlr}[1]{\colorbox{red!20!white}{#1}}
\newcommand{\hlrr}[1]{\colorbox{red!5!white}{#1}}
\newcommand{\hlg}[1]{\colorbox{gray!20!white}{#1}}

\newcommand{\baseline}{\texttt{{Baseline}}}
\newcommand{\slm}{\texttt{{SLM}}}
\newcommand{\rkd}{\texttt{RKD}}
\newcommand{\salt}{\texttt{SALT}}
\newcommand{\saltds}{$\texttt{SALT}_{\texttt{DS}}$}

\def\RCE{R}
\def\RCEN{R_{N}}
\def\thetaest{\hat{\btheta}}
\def\RCENW{R^{\omega}_{N}}
\def\RCEW{R^{\omega}}

\newcommand{\ourtitle}{A Little Help Goes a Long Way: Efficient LLM Training by Leveraging Small LMs}
\title{\ourtitle}

\author{Ankit Singh Rawat\ $^1$ \quad Veeranjaneyulu Sadhanala\ $^1$ \quad Afshin Rostamizadeh\ $^{1}$ \\
Ayan Chakrabarti\ $^{1}$ \quad Wittawat Jitkrittum\ $^{1}$ \quad Vladimir Feinberg\ $^{2}$ \quad Seungyeon Kim\ $^{2}$ \\
Hrayr Harutyunyan\ $^{1}$ \quad Nikunj Saunshi\ $^{1}$ \quad Zachary Nado\ $^{2}$ \quad Rakesh Shivanna\ $^{2}$ \\
Sashank J. Reddi\ $^{1}$ \quad Aditya Krishna Menon\ $^{1}$ \quad Rohan Anil\ $^{2}$ \quad Sanjiv Kumar\ $^{1}$
}
\affil{$^1$ Google Research \quad $^2$ Google DeepMind
}
\date{}

\begin{document}

\maketitle

\begin{abstract}
A primary challenge in large language model (LLM) development is their onerous pre-training cost. Typically, such pre-training involves optimizing a self-supervised objective (such as next-token prediction) over a large corpus. This paper explores a promising paradigm to improve LLM pre-training efficiency \emph{and} quality by suitably leveraging a \emph{small} language model (SLM). In particular, this paradigm relies on an SLM to both (1) provide soft labels as additional training supervision, and (2) select a small subset of valuable (``informative'' and ``hard'') training examples. Put together, this enables an effective transfer of the SLM's predictive distribution to the LLM, while prioritizing specific regions of the training data distribution. Empirically, this leads to reduced LLM training time compared to standard training, while improving the overall quality. Theoretically, we develop a statistical framework to systematically study the utility of SLMs in enabling efficient training of high-quality LLMs. In particular, our framework characterizes how the SLM's seemingly low-quality supervision can enhance the training of a much more capable LLM. Furthermore, it also highlights the need for an \emph{adaptive} utilization of such supervision, by striking a balance between the bias and variance introduced by the SLM-provided soft labels. We corroborate our theoretical framework by improving the pre-training of an LLM with 2.8B parameters by utilizing a smaller LM with 1.5B parameters on the Pile dataset.

\end{abstract}

\section{Introduction}
\label{sec:intro}

Owing to the recent surge in their ever-growing capabilities, large language models (LLMs)~\citep{chowdhery2022_palm, touvron2023_llama, gpt4_techreport, anil2023palm, jiang2023mistral, team2023gemini, anthropic2024claude}, have become the focal point of machine learning research.
Several research efforts focus on either further enhancing LLM performance, 
or utilizing LLMs in novel applications 
ranging from conversational agents/assistants~\citep{thoppilan2022lamda} to novel material design~\citep{rubungo2023llm}.
Highly capable general-purpose LLMs rely on two critical ingredients:
choosing a model architecture with a very large number of parameters~\citep{chowdhery2022_palm, smith2022using},
and \emph{pre-training} this model on a corpus with a very large number of examples~\citep{llama3modelcard, together2023redpajama}.
Due to the large size of model and corpus, the computational cost of pre-training can be highly onerous.
Thus, sustainable advancement and widespread adoption of LLMs hinges on designing novel architectures
and algorithms that can reduce the overall training (particularly, pre-training)
compute cost and improve the data efficiency for LLM development.

This paper focuses on leveraging \emph{small language models} (\emph{SLMs}) for efficient LLM pre-training. Interestingly, a growing literature~\citep[see, e.g.,][]{gupta2024language, chen2023frugalgpt,Yue:2024} shows that, despite their  
limited model capacity, SLMs can acquire a good domain understanding of pre-training data distribution. 
Particularly, SLMs can perform well on a large portion of ``\textit{easy}'' instances, while still providing valuable information towards identifying the remaining ``\textit{hard}'' instances, e.g., via the confidence, margin, or similar measures based on their predictive distribution. This prompts us to explore the following question:
\begin{center}
\vspace{-2mm}
\textit{Can we speed up pre-training of a high-quality \emph{large} LM by transferring the predictive distribution resulting from pre-training of a lower-quality \emph{small} LM?}
\vspace{-2mm}
\end{center}
Note that suitable SLMs are often readily available during LLM development as previous-generation models trained on similar pre-training corpora or smaller models trained for initial exploration around architectural and algorithmic choices on the current pre-training corpora itself. Furthermore, if proven, the potential of SLMs to enhance LLM quality and efficiency, coupled with their relatively cheaper development cost, strongly justifies training such models even to solely aid LLM training.

\textit{Knowledge distillation}~\citep[KD;][]{Bucilua:2006,Hinton:2015}
is a natural candidate to achieve our underlying objective by utilizing the SLM as \textit{teacher} model to transfer its predictive distribution to \textit{student} LLM during pre-training. However, it is unclear if KD can be helpful in realizing our goal as unlike a typical KD setup -- wherein a larger or stronger teacher is used to train a smaller or weaker student -- we are hoping to leverage a smaller and \emph{weaker teacher} LM to improve the pre-training efficiency and quality of a larger and \emph{stronger student} LM.

We begin by developing a statistical framework to study KD in the context of language modeling. We obtain novel risk bounds that delineate the desirable properties of the teacher LM-provided supervision that can enhance the student LM's performance, even when one employs a perceivably weaker SLM as the teacher. To the best of our knowledge, ours are the first such bounds for language modeling. Notably, our bounds subsume standard pre-training as a special case, and control
LM generalization 
as we scale model capacity and
the amount of pre-training data, with latter being measured in terms of either number of training sequences or number of training tokens. We believe that these bounds are of independent interest to broader community working on LLMs.

\begin{figure}[t]
  \centering
  \includegraphics[width=0.7\textwidth]{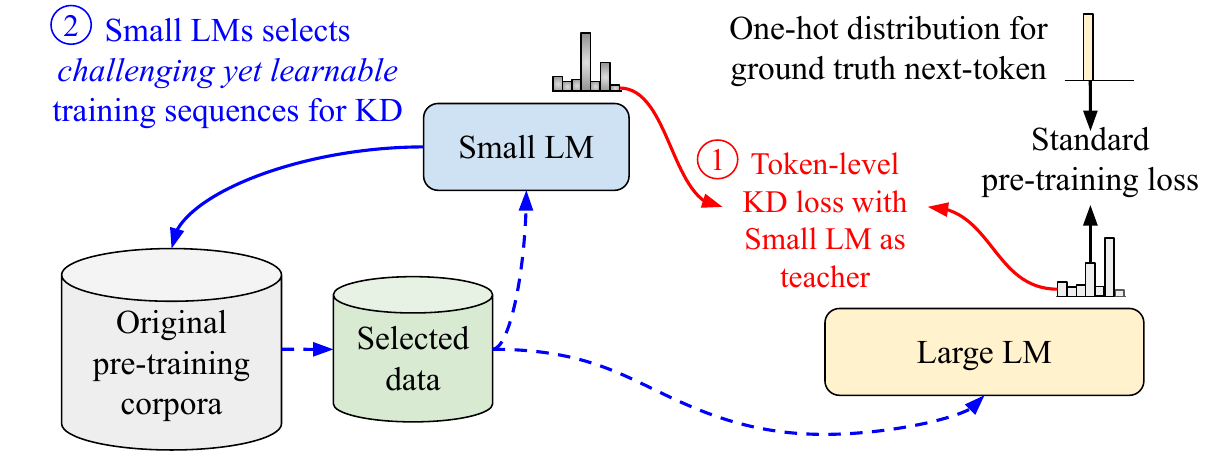}
  \caption{
  An overview of \textit{\underline{s}mall model \underline{a}ided \underline{l}arge model \underline{t}raining} (\salt{}) pre-training.
  \salt{} utilizes an SLM in two 
  ways to improve the pre-training of LLM: {\color{red}\raisebox{0pt}{\textcircled{\raisebox{-0.9pt} {1}}}} To perform KD with SLM as teacher in the early phase of LLM pre-training; and {\color{blue}\raisebox{0pt}{\textcircled{\raisebox{-0.9pt} {2}}}} To obtain a valuable subset of pre-training corpora to be utilized during the KD.
  }
  \label{fig:salt-overview}
\end{figure}

Our statistical analysis lays the foundation for an adaptive pre-training method that leverages SLM via KD only in the so-called ``easy'' regions where SLM can approximate the ground truth next-token distribution well. Combining this with the tendency of neural network to learn easier examples first~\citep{Kalimeris:2019,refinetti23a}, we propose \textit{\underline{s}mall model \underline{a}ided \underline{l}arge model \underline{t}raining} (\salt{}) -- a two-stage pre-training approach that employs KD from an SLM in the early phase of the LLM pre-training and resorts to standard self-supervision-based training for the rest of the pre-training.
We then expand the \salt{} method by employing SLM to additionally perform \textit{data selection} for the KD phase.
Our selection procedure focuses on identifying \textit{challenging} yet \textit{learnable} sequences from the easy region of the data distribution
to ensure an effective transfer of SLM's predictive distribution during KD 
(see Fig.~\ref{fig:salt-overview} for an overview). Our 
empirical study, focusing on both few-shot and post-supervised fine-tuning (SFT) performance, validates the utility of \salt{} for improving both quality and training efficiency of LLMs compared to standard pre-training. Our key contributions are summarized as follows:

\begin{enumerate}[leftmargin=8mm, itemsep=1mm, partopsep=0pt,parsep=0pt,topsep=0pt,label=(\roman*)]
\item We present a statistical framework for KD in the context of language modeling, which provides novel risk bounds describing how even a perceivably weaker teacher LM can improve the quality of a larger student LM (cf.~Sec.~\ref{sec:theory}). 
\item Guided by our 
framework, we propose a two-stage pre-training method,
namely \salt{}, that uses SLMs as teacher to perform KD during the first stage corresponding to early phase of LLM pre-training.
We extend \salt{} by utilizing SLMs to perform data selection for the KD phase, facilitating an effective transfer of predictive distribution from SLM to LLM (cf.~Sec.~\ref{sec:method}).

\item We showcase the utility of \salt{} (with and without data selection) by training a 2.8B parameter LM with the help of 1.5B parameter LM on the Pile dataset~\citep{gao2020thepile}. The 2.8B LM trained with \salt{} outperforms a 2.8B LM trained via standard pre-training method on a wide range of popular few-shot benchmarks while utilizing less than $0.7$X training step budget, resulting in $\sim 28\%$ wall-clock time reduction.
Moreover, \salt{} models consistently demonstrate significant downstream performance gains after SFT on multiple domains 
(cf.~Sec.~\ref{sec:exp}).
\end{enumerate}

\section{Background}\label{sec:bg}

\textbf{Language modeling.} Given a large corpus, language modeling aims to train a model that can assign probabilities or likelihood to each sequence $\mathbf{x} \in \VC^{\star}$, where $\VC$ denotes the underlying vocabulary with $V = |\VC|$ tokens. Assuming that the language model (LM) is parameterized by $\btheta$, 
it assigns the following probability to a $T$-token long input sequence $\rvx = [x_1, x_2,\ldots, x_T]$:
\begin{align*}
    P_{\btheta}(\rvx) = P_{\btheta}(x_1)P_{\btheta}(x_2|x_1)\cdots P_{\btheta}(x_T|x_1,\ldots, x_{T-1}).
\end{align*}

Transformers~\citep{vaswani2017attention} are the most prominent architecture supporting LMs~\citep{gpt4_techreport,touvron2023_llama,team2023gemini}, which we briefly discuss in Appendix~\ref{appen:lm-transformers}.

\textbf{Standard LM pre-training.} Typically, LM pre-training involves the \emph{next-token prediction} task:
given a training sequence $\rvx = [x_1, x_2,\ldots, x_T]$, 
for each $t \in [T]$, one maximizes the log-likelihood $\log P_{\btheta}(x_t | \rvx_{\leq t -1})$.
This amounts to \textit{minimizing} the \textit{cross-entropy} loss between the per-token LM prediction distribution $P_{\btheta}(\cdot | \rvx_{\leq t -1})$ and the one-hot distribution\footnote{For $v \in \VC$, we define $\mathds{1}_{x}(v) = 1$ if $v = x$ and $\mathds{1}_{x}(v) = 0$ if $v \neq x$.} $\mathds{1}_{x_{t}}(\cdot)$ 
defined by the \textit{ground truth} next-token $x_{t}$.
Thus, the overall loss associated with 
$\rvx$ becomes 
\begin{align}
\label{eq:std_loss_ce}
    \ell(\rvx; \theta) =  {1}/{T}\cdot\sum\nolimits_{t\in[T]} - \log P_{\btheta}(x_t | \rvx_{\leq t -1}) = {1}/{T}\cdot\sum\nolimits_{t\in[T]} \mathsf{CE}\big(\mathds{1}_{x_{t}}(\cdot), P_{\btheta}(\cdot | \rvx_{\leq t -1})\big),     
\end{align}
where 
$\mathsf{CE}(P_1, P_2)$=$-\sum_{v \in \VC}P_1(v)\log P_2(v)$
is the cross-entropy between
distributions $P_1$ and $P_2$.

\textbf{Knowledge distillation for LM.} Going beyond the ground truth next-token based
loss in \eqref{eq:std_loss_ce}, one can utilize the per-token prediction distribution provided by another LM, say the one parameterized by $\bzeta$, as additional
supervision. Formally, given the context $\rvx_{\leq t -1}$, one can train the LM parameterized by $\btheta$ via aligning its prediction distribution $P_{\btheta}(\cdot | \rvx_{\leq t -1})$ with $P_{\bzeta}(\cdot | \rvx_{\leq t -1})$.
KL divergence is a common choice to promote such an alignment,
which amounts to minimizing the following cross-entropy loss:
\begin{align}
\label{eq:tkd_loss_ce}
    \ell(\mathbf{x}; \bzeta \to \btheta) &=  {1}/{T}\cdot\sum\nolimits_{t\in[T]} \mathsf{CE}\big( P_{\bzeta}(\cdot | \rvx_{\leq t -1}),  P_{\btheta}(\cdot | \rvx_{\leq t -1})\big).
\end{align}

Training based on the loss in \eqref{eq:tkd_loss_ce}
is referred to as the {\em token-level knowledge distillation}~\citep[KD;][]{kim2016sequence} in the literature, with LMs parameterized by $\bzeta$ and $\btheta$ 
termed as the teacher and student LMs, respectively. See~Appendix~\ref{append:kd-variants} for a discussion on other related KD for LM variants.
\textit{Temperature scaling} of teacher is a common strategy~\citep{zheng2024knowledge} where, given a temperature $\rho > 0$, one utilizes
$
P_{\bzeta, \rho}(\cdot | \rvx_{\leq t -1}) = {P_{\bzeta}(\cdot| \rvx_{\leq t -1})^{\rho}}/{\sum_{v' \in \VC}P_{\bzeta}(v' | \rvx_{\leq t -1})^{\rho}}
$
during KD, resulting in the loss:
\begin{align}
\label{eq:tkd_loss_ce_temp}
    \ell(\mathbf{x}; \bzeta^{\rho} \to \btheta) =  {1}/{T}\cdot\sum\nolimits_{t\in[T]} \mathsf{CE}\big( P_{\bzeta, \rho}(\cdot | \rvx_{\leq t -1}),  P_{\btheta}(\cdot | \rvx_{\leq t -1})\big).
\end{align}
In practice, one typically utilizes both ground truth next-token as well as teacher's next-token distribution and, for a \textit{distillation loss weight} $\omega \in [0,1]$, minimizes
the following as the loss for
$\rvx$:  
\begin{align}
\label{eq:tkd_final}
\ell^{\omega}(\rvx; \theta) \triangleq
(1 - \omega)\cdot\ell(\rvx; \btheta) + \omega \cdot \ell(\mathbf{x}; \bzeta^{\rho} \to \btheta).
\end{align}
Note that, for brevity, our notation $\ell^{\omega}(\rvx; \theta)$ omits the dependence on $\zeta$ and $\rho$.

\section{Theoretical analysis: When can KD help language modeling?}
\label{sec:theory}

As alluded in the introduction, we hope to leverage KD with an SLM as the teacher to speed-up the pre-training of a high quality LLM. However, due to SLM's relatively limited capacity and inferior quality, it is not immediately clear that such a teacher can benefit the LLM. Motivated by this, we now develop a rigorous statistical framework for KD in the context of language modeling by building 
on the works of ~\citet{menon21statistical,Dao:2021,Ren:2022}.
Novel risk bounds originating from this framework highlight how a teacher LM -- even a perceivably weaker one -- can benefit student LLM by striking the right balance in terms of a bias-variance trade-off. 

Notably, our analysis allows one to control the generalization gap for the student LM in terms of both number of training sequences $N$ 
as well as number of total tokens $NT$, with the latter being highly non-trivial due to possibly arbitrary dependence within a training sequence. In this work, we crucially leverage certain natural stability conditions on the underlying distribution and function class to obtain such
bounds in terms of $NT$. Next, we setup some necessary notation and then present our risk bounds as functions of $N$ and $NT$ in Sections~\ref{sec:erb-seq-level} and \ref{sec:erb-token-level}, respectively. Sec.~\ref{sec:bias-var} utilizes our bounds to justify the utility of SLMs for improving LLM model quality via KD.

Let $\DCal$ be the
data distribution that generates 
$N$ independent training sequences $\SCal_{N} = \{\rvx^{(i)}\}_{i \in [N]} \subset \VC^{T}$, i.e., $\rvx^{(i)} \sim \DCal,~\forall~i \in [N]$.\footnote{{
Our analysis can be extended to \textit{varying} length sequences at the cost of increased notational complexity.}} 
Given 
$\SCal_N$ and CE surrogate loss
(cf.~\eqref{eq:std_loss_ce}),
the \textit{empirical surrogate risk}, i.e., standard training objective, and its population version
for an LM parameterized by $\btheta$ are:
\begin{equation}
\label{eq:er-pr}
\RCEN(\btheta) 
= {1}/{N}\cdot\sum\nolimits_{\rvx \in \SCal_N}\ell(\rvx; \btheta) 
\quad \text{and} \quad \RCE(\btheta) = \mathbb{E}_{\rvx \sim \DCal}\big[\ell(\rvx; \btheta)\big].
\end{equation}
On the other hand, the empirical surrogate risk for KD, i.e., the KD training objective, and its population version take the following form (note that we omit dependence on $\bzeta$ and $\rho$):
\begin{align}
\label{eq:erm_distill}
\RCENW(\btheta)
=  {1}/{N}\cdot\sum\nolimits_{\rvx \in \SCal_N}\ell^\omega(\mathbf{x}; \btheta) \quad \text{and} \quad
\RCEW(\btheta)
= \mathbb{E}_{\rvx \sim \DCal}\big[\ell^\omega(\mathbf{x};\btheta)\big]. 
\end{align}

\subsection{Excess surrogate risk bound for LM in terms of number of sequences}
\label{sec:erb-seq-level}
Given a potentially \textit{infinite} function class $\Theta$ for student LMs, 
let $\thetaest$ and $\btheta^{*}$ be the minimizers of KD training objective $\eqref{eq:erm_distill}$ and the population
risk in \eqref{eq:er-pr}, respectively:  
\begin{align}
\label{eq:kd-erm-theta}
\thetaest 
:= \argmin\nolimits_{\btheta \in \Theta} \RCENW(\btheta) \quad \text{and} \quad \btheta^{*} = \argmin\nolimits_{\btheta \in \Theta} \RCE(\btheta).
\end{align}
We want to compare the test performance (population risk) of $\thetaest$ with that of $\btheta^{*}$ -- the optimal LM in $\Theta$.
Before stating our excess risk bound, we introduce an assumption that we rely on in our analysis.
\begin{assumption}
\label{assm:loss_bound}
The per-token log-loss with at most $T$-token long sequences for the function class $\Theta$ is bounded by $M$, i.e.,
$\sup_{\btheta \in \Theta; \rvx \in \VC^{\leq T-1}} \max_{v \in \VC}\vert \log P_{\btheta}(v| \rvx) \vert \leq M.$
\end{assumption}

\begin{remark}
Assuming loss to be bounded is a standard practice in the statistical learning literature which can be realized, e.g., by clipping the CE loss by a sufficiently large constant $M$. Recently, \citet{lotfi24-seq-bound}
realized such an assumption by adding a small amount of uniform noise to LM predictions in their study on generalization for LMs trained via standard pre-training \textit{without} KD.
\end{remark}

We are now ready to present our excess risk bound.
Due to the page limit
we provide an \textit{informal} statement of our result below; see Appendix~\ref{appen:esr} for the complete statement and its proof.

\begin{theorem}[Informal]
\label{thm:esr}
Let $\thetaest$ and $\btheta^{*}$ be as defined in \eqref{eq:kd-erm-theta}. 
Define $f^\btheta: \VC^T \to [0, M]$ by $f^\btheta(\rvx) = \ell^\omega(\mathbf{x}; \btheta), ~\forall \rvx \in \VC^T, \btheta \in \Theta.$
Then, under Assumption~\ref{assm:loss_bound}, 
with probability at least $1 - \delta$,
we have
\begin{align*}
&\RCE(\thetaest ) - \RCE(\btheta^{*})
 \leq
  \frac{c_1}{\sqrt{N}}\cdot \Big(\sqrt{V_N(f^{\thetaest})\log\left({2\MC(N)}/{\delta}\right)} + \sqrt{V_N(f^{{\btheta}^{*}})\log\left({4}/{\delta}\right)}\Big) \nonumber \\
 &\qquad + (4M \omega)/T \cdot \sum\nolimits_{t\in[T]} \mathbb{E}\left[\mathsf{D}_{\rm TV}\Big(P_{\bzeta,\rho}(\cdot|\rvx_{\leq t-1}), \DCal(\cdot|\rvx_{\leq t-1})\Big)\right] \; + {c_2 M}/{(N-1)}\cdot\log(\MC(N)/\delta),
\end{align*}
where $\mathsf{D}_{\rm TV}$ is TV distance,
$
    V_N(f^{\btheta}) = \frac{1}{N(N-1)}\sum_{1 \leq i < j \leq N} \big(f^{\btheta}(\rvx^{(i)}) - f^{\btheta}(\rvx^{(j)})\big)^2 
$
is sample variance,
$\MC(N)$
depends on the
growth function of $
\big\{ f^\btheta : \btheta \in \Theta\big\}$,
and $c_1$ \& $c_2$ are universal constants.
\end{theorem}

\subsection{Excess surrogate risk bound for LM in terms of number of tokens}
\label{sec:erb-token-level}
For an LM $\btheta \in \Theta$
and a training sequence $\rvx = [x_1, x_2,\ldots, x_T] \in \VC^T$, define
\begin{align}
\label{eq:lm-stability}
\xi_t(\rvx; \btheta)
= \mathbb{E}_{\rvz\sim \DCal} \left[\ell^{\omega}(\mathbf{z}; \btheta) | 
        \mathbf{z}_{\leq t - 1} = \rvx_{\le t-1}\right] - 
  \mathbb{E}_{\rvz\sim \DCal}\left[\ell^{\omega}(\mathbf{z}; \btheta) | 
        \mathbf{z}_{\leq t} = \rvx_{\le t}\right], \quad t \in [T].
\end{align}
Note that $\xi_t(\rvx; \btheta)$ does not depend on $\rvx_{>t}$.
For $t \in [T]$, $\xi_t(\rvx; \btheta)$ 
measures the expected KD loss deviation for the student 
when we condition on
the context up to $(t-1)$-th vs. $t$-th token, respectively, and sample the remaining tokens 
from $\DCal$. In general,
the deviation could be large as changing a \textit{single} token in the context can significantly alter LM's distribution.
However, 
a well-behaved LM should be robust to such perturbations.
Motivated by this, we introduce the following assumption.

\begin{assumption}
\label{assm:xi}
Given the data distribution $\DCal$ and \textit{finite} function class $\Theta$, the following holds for any $\rvx \in {\rm Support}(\DCal)$, $\btheta \in \Theta$, and $t \in [T]$:
\begin{align}
|
\xi_{t}(\rvx; \btheta)
| & \leq C_t \leq C; \label{eq:assm_xi1} 
\quad \text{and} \quad 
\mathbb{E}\left[
\xi^2_{t}(\rvx; \btheta)
| \rvx_{\leq t-1}\right] \leq V_t. 
\end{align}
\end{assumption}
Under Assumption~\ref{assm:xi}, we obtain the following result on student's generalization gap. For proof, please see Appendix~\ref{sec:token_level_bound_pf}.

\begin{theorem}
\label{thm:token-level-generalization-bound}
Let $\Theta$ be a finite function class. Under Assumption~\ref{assm:xi}, with probability at least $1 - \delta$, the following holds for the student LM $\thetaest \in \Theta$ obtained via KD:
\begin{align}
&\RCE(\thetaest) \leq  \RCENW(\thetaest) 
+  \sqrt{2{\sum\nolimits_{t}{V_t}/{N}} \cdot\log\left({|\Theta|}/{\delta}\right)} \; + {2C}/({3N})\cdot\log\left({|\Theta|}/{\delta}\right) \; + \nonumber \\
&\qquad \qquad (4M \omega) / T \cdot \sum\nolimits_{t\in [T]} \mathbb{E}_{\rvx_{\leq t-1} \sim \DCal}\mathsf{D}_{\rm TV}\big(P_{\bzeta,\rho}(\cdot|\rvx_{\leq t-1}), \DCal(\cdot|\rvx_{\leq t-1})\big).
\end{align}
\end{theorem}

\begin{remark}[Dependence of $C,\{V_t\}$ on $T$]
\label{rem:graceful}
Theorem~\ref{thm:token-level-generalization-bound} captures a fine-grained dependence on $C$ and $\{V_t\}$ from Assumption~\ref{assm:xi}. For a robust LM, when $C$ is $\mathcal{O}(1/T)$ and $\{V_t\}$ are $\mathcal{O}(1/T^2)$ (note the scaling of $\ell^\omega$ by $T$), we get the tightest generalization gap decaying with ${NT}$. For a non-robust LM in the worst case, $C$ and $\{V_t\}$ can be as large as $\mathcal{O}(1)$. The bound is not tight anymore and we fall back to the bound in Theorem~\ref{thm:esr} that only decays with ${N}$ instead of ${NT}$.
\end{remark}

\begin{remark}~Recently, \citet{lotfi2024-token-bound} obtained generalization bounds for LLMs trained \textit{without} KD 
in terms of total number of tokens. However, even when we specialize 
Theorem~\ref{thm:token-level-generalization-bound} to standard pre-training by setting 
$\omega$ to $0$, our result significantly differs from theirs both in terms of proof technique as well as its implications. Crucially,
results in \citet{lotfi2024-token-bound} only holds for the contexts seen during training whereas our bound 
enables controlling LM's generalization gap 
on novel contexts generated from the data distribution during the test time.
\end{remark}

In a concurrent work, \citet{zekri2024large} provide novel generalization bounds for language models by drawing connections between auto-regressive models and  finite state Markov chains. Notably, their work does not explore KD in language modeling context.

\subsection{KD out-performing standard pre-training: A bias-variance trade-off}
\label{sec:bias-var}
Empowered by our novel risk bounds, we now provide justification for why KD can outperform standard pre-training.
Specifically, based on Theorem~\ref{thm:token-level-generalization-bound}, we clarify 
when KD might lead to a tighter bound than standard pre-training. Similar conclusions follow from Theorem~\ref{thm:esr} by extending the arguments from \citet{menon21statistical} to
language modeling. As per Theorem~\ref{thm:token-level-generalization-bound}, 
three key quantities control the generalization 
of the student:
(1) $\sum_{t}V_t$ which relates to loss variance; (2) $C$ which relates to extreme loss values;
and (3) divergence between the teacher-provided
distribution and the ground truth
distribution:
$\textsc{Div}(\zeta, \omega) = \omega \cdot \sum_{t} \mathbb{E}\left[\mathsf{D}_{\rm TV}\left(P_{\bzeta,\rho}(\cdot|\rvx_{\leq t-1}), \DCal(\cdot|\rvx_{\leq t-1})\right)\right].$
Under Assumption~\ref{assm:loss_bound},
only $\sum_{t}V_t$ and $\textsc{Div}(\zeta, \omega)$ are
crucial in distinguishing KD and standard pre-training.

Since $\textsc{Div}(\zeta, 0) = 0$, one may surmise that standard pre-training (i.e, $\omega=0$) leads to tighter bound.
But as we detail in Appendix~\ref{appen:kd-std-comparison} due to the page limit, the variance term becomes smaller as we \textit{increase} $\omega$.
Thus, with a careful selection of 
$\omega$, the \textit{variance reduction} via KD can offset the \textit{bias}
$\textsc{Div}(\zeta, \omega)$. In particular, if the teacher 
closely approximates the ground truth
distribution
so that the bias $\textsc{Div}(\zeta, \omega)$ is small even for large $\omega$, then the variance reduction via KD becomes prominent, resulting in significantly tighter generalization gap compared to standard pre-training.

\textbf{Performance gain via KD from an SLM as teacher.}~While small teacher LMs -- the main interest of this work -- also lead to variance reduction, they are typically not powerful enough to model the true distribution over the entire data domain very well. Thus, any effect of variance reduction
via KD with such a teacher would be washed away by the large bias $\textsc{Div}(\zeta, \omega)$.
This highlights the need for an adaptive form of KD from SLMs.
Even SLMs with their limited capacity can approximate the 
true distribution well on certain regions of the data domain,
which we call the `easy' regions. 
Thus, one can employ KD from
SLMs on the easy regions to benefit from the variance reduction
without incurring large bias and guarantee improved student LLM performance on these regions.
For the remaining (`hard') regions, where the biasx
is large enough to overshadow the contributions of variance reduction, one should 
not perform KD from
SLMs and utilize the standard pre-training loss.

\section{\salt{}: \underline{S}mall model \underline{a}ided \underline{l}arge model \underline{t}raining}
\label{sec:method}
\begin{algorithm}[!t]
    \caption{Small model aided large model training (\salt)}\label{alg:stl-two-stage}
    \begin{algorithmic}[1]
        \Require Training data $\SCal_{N} = \{\rvx^{(i)}\}_{i \in [N]} \subset \VC^{T}$, gradient-based optimization algorithm $\AC$, SLM parameterized by $\bzeta$, distillation loss weight $\omega \in [0, 1]$, teacher temperature $\rho > 0$, batch size $B$, training step budget $n$, learning rate schedule $\{\eta_j\}_{j \in [n]}$, and $n_{\rm KD} \leq n$.
        \Ensure Pre-trained LLM parameterized by $\hat{\btheta} \in \Theta$. 
        \State Initialize $\btheta_0 \in \Theta$.
        \For {$j = 1,2,\ldots, n_{\rm KD}$} \qquad \qquad \qquad \;\;\; \texttt{\color{gray}// First stage of LLM pre-training via KD}
            \State Construct a new batch of $B$ training sequences $\BCal_{j} = \{\rvx^{(i)}\}_{i \in [B]} \subset \SCal_{N}$.
            \State Update $\theta_{j+1}$ with step size $\eta_j$ via one step of $\AC$ on $\LCal^{\rm KD}(\BCal_j) = \frac{1}{B}\sum_{\rvx \in \BCal_j}\ell^\omega(\mathbf{x}; \btheta_j)$.
        \EndFor
        \For {$j =n_{\rm KD}+1, n_{\rm KD}+2,\ldots, n$} \qquad \texttt{\color{gray}// Second stage:~standard training}
            \State Construct a new batch of $B$ training sequences $\BCal_{j} = \{\rvx^{(i)}\}_{i \in [B]} \subset \SCal_{N}$.
            \State Update $\theta_{j+1}$ with step size $\eta_j$ via one step of $\AC$ on $\LCal^{\rm Std}(\BCal_j) =  \frac{1}{B}\sum_{\rvx \in \BCal_j}\ell(\rvx; \btheta_j)$.
        \EndFor
        \State $\hat{\btheta}\leftarrow\btheta_n$
    \end{algorithmic}
\end{algorithm}

We now operationalize the key takeaway from
Sec.~\ref{sec:theory}
by proposing
\salt{} -- a simple yet effective two-stage pre-training method.
\salt{} relies on the inherent preference of a model to first focus on easier
supervision before fitting more complex supervision during
training~\citep{Kalimeris:2019,refinetti23a}
to perform selective KD from a teacher SLM.
We then expand \salt{}
to perform \textit{explicit selection} of
training sequences where we want to
conduct KD from an SLM on. In particular, we identify \textit{challenging}
sequences which are \textit{learnable}
as per teacher SLM; as a result, performing KD on those sequences
result in further variance reduction~\citep{katharopoulos18a}.

\noindent\textbf{Two-stage LLM pre-training via}~\salt{}\textbf{.}~
Inspired by our analysis,
we propose a two-stage pre-training methods for LLMs in Algorithm~\ref{alg:stl-two-stage}. The
algorithm employs KD with SLM as a teacher in the first stage which lasts till $n_{\rm KD}$ training steps,
and transitions to standard pre-training \emph{without} KD in the second stage. Note that we are interested in the selective transfer of predictive distribution from teacher SLM to student LLM in those regions where SLM performs well by capturing true distribution. By design, KD aims to align predictive distributions of the teacher and student. On the regions where SLM performs well, we expect it to exhibit reasonably confident predictive distribution that should align with ground truth next-token~\citep{gupta2024language}, thereby constituting an easier supervision signal for the LLM.
In contrast, on hard instances where SLM's predictive distribution is not confident enough or does not align well with the ground truth next-token, learning will be delayed to the later phase of the training~~\citep{Kalimeris:2019,refinetti23a}.
Thus, \salt{} relies on the tendency of neural networks to focus on easier instances early during the training to perform desirable knowledge transfer from SLM in the first stage.
Once the student LLM is sufficiently aligned with teacher SLM on easier regions, it starts utilizing its model capacity to further align with SLM on more complex regions where high divergence between SLM and ground truth distribution can become detrimental to the 
LLM's performance. Switching to standard pre-training based solely on the self-supervision from next-tokens in the second stage prevents this undesirable over-alignment.
We empirically verify the above intuition behind \salt{} in Sec.~\ref{sec:performance_slicing_on_example_hardness}.

\noindent\saltds\textbf{:}~\salt{} \textbf{with \underline{d}ata \underline{s}election.}~
We now endow \salt{} approach (cf.~Algorithm~\ref{alg:stl-two-stage}) with explicit selection of examples where we want to transfer teacher SLM's predictive distribution on, with SLM itself enabling the data selection. In particular, we want to select the most informative (or challenging) examples among the ones that SLM performs well on.
Towards this, given the SLM 
$\bzeta$ and a positive integer $k$, we assign a score $S_{\bzeta, k}(\rvx)$ to training sequence $\rvx$, with a higher score indicating a higher likelihood to be included for training. 
More specifically, we compute the per-token cross-entropy loss of SLM 
on $\rvx$ and
aggregate (typically using the median) into a 
sequence-level score. This encourages selecting more challenging examples. However, in the spirit of selecting examples that are still \emph{learnable},
we remove all losses where the ground-truth token is not in top-$k$ outputs of the model before aggregating:
\begin{align}
\label{eq:ds_score}
   S_{\bzeta, k}(\rvx)
    = \mathrm{median} \Big(\Big\{
    -\mathds{1}\{x_i \in \mathrm{argtop}_k(P_\bzeta(\cdot|\rvx_{<i})) \} 
    \log P_{\bzeta}(x_i | \rvx_{<i}):
    i \in [n]
    \Big\}\Big),
\end{align}
where $\mathrm{argtop}_k(P_\bzeta(\cdot|\rvx_{<i}))$ denotes the top-$k$ scoring tokens at position $i$ according to SLM. 
Given a scored pool of sequences, we select the top-$m$ scoring sequences where $m$ is sufficiently large to complete the first stage of
Algorithm~\ref{alg:stl-two-stage}.
The reader may note, that if the SLM has been trained with the same dataset that it is scoring, then the computed per-token loss (and resulting sequence score) may be heavily biased. To circumvent that, we use an ``early checkpoint'' of the model $\bzeta_{n_0}$, which has trained on a small number of examples from the overall training set $n_0 \ll N$.  We then sample from the remainder of the training examples using score $S_{\bzeta_{n_0}, k}(\cdot)$. Although the early model $\bzeta_{n_0}$  
may be a lower quality model due to training with relatively little data, it is only the relative ordering of examples that is important when computing a score, rather than the absolute score.

\section{Experiments}
\label{sec:exp}
We now showcase the potential of leveraging  SLMs for improving LLM pre-training, by realizing both better quality and improved training efficiency. 
We demonstrate the additive utility of two aspects of our proposal: 
(1) employing SLMs as teacher models during the early phase of LLM pre-training (\salt{}); and 
(2) further performing data selection via SLMs during the KD phase (\saltds{}). 

Throughout our study, we compare with a natural baseline (denoted \baseline{}) where the large LM is pre-trained in a standalone manner with a self-supervised objective over a pre-training set. This enables us to fairly compare our proposed approach as we fix the model architecture and the underlying pre-training data in our evaluation. The key takeaways from this section are:

\begin{itemize}[leftmargin=6mm, itemsep=1mm, partopsep=0pt,parsep=0pt]
\item[(1)] \salt{} and \saltds{} attain \baseline{} performance with less than 70\% of training steps, and significantly outperform \baseline{} with the same number of training steps~(cf.~Sec.~\ref{sec:exp-results}). Additionally, \salt{}
can leverage improved quality small teacher LM to further improve the performance of the large student LM (cf.~Appendix~\ref{sec:exp-ablation})
\item[(2)] \salt{} and \saltds{} pre-trained LMs consistently outperform \baseline{} after SFT on arithmetic reasoning, summarization, and natural language inference (NLI) tasks~(cf.~Sec.~\ref{sec:exp-sft}).
\item[(3)] Step transition from KD phase to standard training phase in \salt{} (cf.~Algorithm~\ref{alg:stl-two-stage}) constitutes a good design choice as it outperforms other natural alternatives~(cf.~Appendix~\ref{sec:exp-ablation}).
\end{itemize}
We also compare \salt{}
with \rkd{} (standing for  \textit{reverse KD}) where we perform KD with \slm{} as the teacher \textit{throughout} pre-training.
\rkd{} results in sub-par performance even compared to \baseline{} as the relatively poor quality of \slm{} limits LLM performance
during the later part of training.

\subsection{Experimental Setup}
\label{sec:exp-setup}

 \begin{figure}
    \centering
    \includegraphics[width=0.45\textwidth]{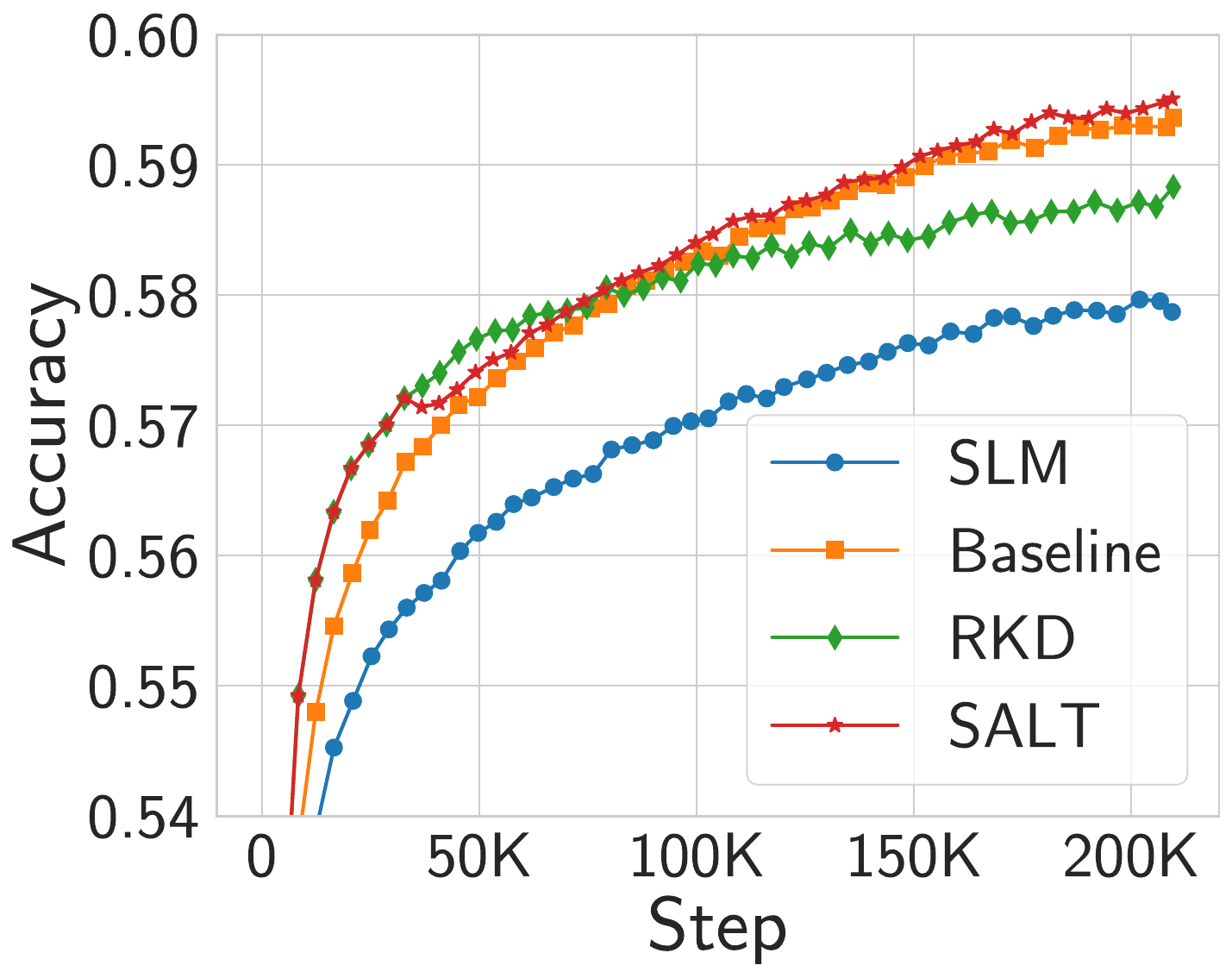}
    \caption{
    Fraction of correct next-token predictions
    for various LMs during training, on a subset of the Pile training set.
    }
    \label{fig:training-curves}
\end{figure}

\textbf{Model architectures and pre-training data.}
We work with standard decoder-only Transformer-based LMs. Our small model (\slm{}) has 1.5B parameters and our large model has 2.8B parameters.
We use a SentencePiece tokenizer~\citep{kudo2018sentencepiece} from \cite{du2022glam} with a vocabulary size of 256$K$. 
We pre-train all LMs on the Pile dataset~\citep{gao2020thepile} for 545 billion tokens 
via {UL2 objective}~\citep{tay2023ul2} with a mixture of causal LM, prefix LM and span corruption tasks.
We use a batch size of 2048 and input sequence length of 1280.
Based on our hyperparameter search, we set the distillation weight $\omega=0.667$ and teacher temperature $\rho=0.25$.
For \saltds{}, we score examples with \slm{} trained for $n_0=26$K steps and set $k=10$ in \eqref{eq:ds_score}.
Further details on model architecture and pre-training are provided in Appendix~\ref{sec:exp-setup-details}.

\textbf{Few-shot evaluation tasks.}~Following the literature~\citep[see, e.g.,][]{anil2023palm,touvron2023_llama},
we perform few-shot evaluation of pre-trained LMs on a wide range of standard benchmarks. We focus on English benchmarks as the pre-training data is mostly English.
We defer the full list of benchmarks to Appendix~\ref{appen:few-shot-benchmarks} which can be
categorized into: (1) world knowledge, (2) reading comprehension, (3) commonsense reasoning, (4) natural language generation (NLG), and (5) SuperGLUE. 
We also consider LAMBADA~\citep{paperno2016lambada} and MBPP~\citep{austin2021program} which are Cloze and code generation tasks, respectively.
We conduct $1$-shot evaluation for all benchmarks, except for MBPP which is $3$-shot. For each benchmark, we report the corresponding prevalent metric in the literature. See Appendix~\ref{app:more_experiments} for details.

\textbf{Fine-tuning tasks.}
We focused on (arithmetic) reasoning, summarization, and NLI tasks where the few-shot performance of all pre-trained models (including baseline and our models) were poor. For arithmetic reasoning, we utilize GSM8K~\citep{nie2020-adversarial}. For summarization, we employ
XSum~\citep[][]{narayan2018dont_xsum} and CNN/DailyMail
~\citep{nallapati2016-abstractive}. For NLI tasks, we employ
ANLI-R1, ANLI-R2, and ANLI-R3
~\citep{nie2020-adversarial}.

\subsection{Results: \salt{} enables efficient training of high quality LLMs}
\label{sec:exp-results}
Interestingly, KD from the seemingly weaker \slm{} does improve LLM training in the beginning compared to \baseline{}, as reflected in the next-token prediction accuracy over the training set in Fig.~\ref{fig:training-curves}. However, continuing KD from the weaker teacher eventually
become detrimental. As evident in Fig.~\ref{fig:training-curves} and Tab.~\ref{tab:validation_set_metrics},
\rkd{} significantly underperforms \baseline{} on both training and validation set. In contrast, \salt{} 
leverages KD from \slm{} only during
first $n_{\rm KD}$ training steps
(cf.~Algorithm~\ref{alg:stl-two-stage}).

\begin{table}[t]
    \caption{\textbf{Accuracy and log-perplexity on a held-out 
    set of 
    Pile.}
    We evaluate the models at an early and the final checkpoint.
    At the end of pre-training ($208$K steps), both \salt{} and \saltds{} improve upon \baseline{} in terms of both next-token prediction accuracy and log-perplexity.
    \rkd{} is \textit{identical} to \salt{} during its first stage, hence they have the same performance at $n_{\rm KD}=36$K steps.
    }
    \label{tab:validation_set_metrics}
\centering
\setlength{\tabcolsep}{1pt}
\scalebox{0.96}{
\begin{tabular}{@{}lccc@{}}
\toprule
\bf Model &  \parbox[c]{2.5cm}{\centering {\bf Evaluation stage \\ (steps)} } &  {\bf Accuracy} ($\uparrow$) &  \parbox[c]{2cm}{\centering {\bf Log ($\downarrow$) \\ perplexity} } \\
\midrule
 \slm{}
 & \centering {Final (208K)} &  57.70 &   1.951 \\
\midrule
\baseline   & \centering \multirow{4}{*}{Early (36K)} &  56.68 &   2.011 \\
\rkd  &  & 57.21 &   2.160 \\
\salt &  & 57.21 &   2.160 \\
\saltds &  & 56.47 &   2.188 \\
\midrule
\baseline & \centering\multirow{4}{*}{Final (208K)} &  58.99 &   1.868 \\
\rkd &   & 58.46 &   2.071 \\
\salt &  & \underline{59.10} &   \underline{1.863} \\
\saltds    & & \textbf{59.17} &   \textbf{1.857} \\
\bottomrule
\end{tabular}}
\end{table}

\noindent\textbf{Quality improvements via}~\salt{}\textbf{.}~Unlike \rkd{}, \salt{} (with $n_{\rm KD}=36$K steps) yields a pre-trained LLM that improves upon \baseline{} on both training set (cf.~Fig.~\ref{fig:training-curves}) and held-out validation set (cf.~Tab.~\ref{tab:validation_set_metrics}).
Tab.~\ref{tab:evals_all_grouped} presents domain-wise few-shot performance of \baseline{}, \rkd{}, \salt{}, and \saltds{} (see Tab.~\ref{tab:evals_grouped_full} Appendix~\ref{app:more_experiments} for
per-task performance
).
Both \salt{} and 
\saltds{} consistently outperform \baseline{} (as well as \rkd{}) at the end of 
training%
, i.e., $@100$\% steps or $208$K steps. In particular, \saltds{} outperforms \baseline{} in $6$ out of $7$ domains as well as the overall average; similarly, \salt{} improves upon \baseline{} in $5$ out of $7$ domains as well as the overall average,
establishing the utility of \salt{} approaches in successfully 
leveraging SLMs to boost the quality of 
LLMs.

\noindent\textbf{Training efficiency realized via}~\salt{}\textbf{.}~As per Tab.~\ref{tab:evals_all_grouped}, \salt{} surpasses \baseline{} at $146$K steps on average few-shot performance, suggesting a savings of $30$\% training compute cost. While 
$n_{\rm KD} = 36$K out of those $146$K steps involve KD which is typically computationally costlier than standard training, as we argue next, the fact that our teacher 
is a SLM ensures that we still realize 
efficiency gains via \salt{}. In particular, the additional compute cost in KD is one forward pass of \slm{} per training step. As a rule of thumb, a forward pass constitutes $1/4$th cost of a training step (which comprises both forward and backward passes). In our setup, a forward pass of \slm{} is approximately half as expensive as that for the LLM. Thus, KD from \slm{} adds a factor of $1/8$ to the cost of the training step of the standard training. 
(Our implementation verifies this as we observe $12.0$\% slow down during KD compared to standard training.) Since KD lasts for $n_{\rm KD} = 36$K out $146$K training steps, the training cost required by \salt{} to surpass \baseline{} is approximately equivalent to that of $(146 + 36/8)$K steps of the standard training. This translates to an efficiency gain of $\sim 28$\% compared to the $208$K steps taken by \baseline{}.

\begin{table}
  \caption{\textbf{Domain-wise few-shot performance of pre-trained LMs.}
  \salt{} and \saltds{} already outperform \baseline{} in terms of average few-shot performance at $70$\% of their training step budget, thereby improving both training efficiency and model quality. \rkd{} (i.e., naively distilling from \slm{} throughout pre-training) performs much worse than \baseline{}. The best and second-best results for each domain are \textbf{boldfaced} and \underline{underlined}, respectively.
  }
  \label{tab:evals_all_grouped}
  \centering
   \setlength{\tabcolsep}{2pt}
   \scalebox{0.96}{
    \begin{tabular}{lccccccccccccc}
    \toprule
    \bf Domain & \textbf{\# Tasks} && \slm && \baseline && \rkd &&
    \multicolumn{2}{c}{\salt} &&
    \multicolumn{2}{c}{\saltds} \\
    \cmidrule{6-6} \cmidrule{8-8}  \cmidrule{10-11} \cmidrule{13-14}
    &  &&  && \textit{@100\%} && \textit{@100\%} &&
    \textit{@70\%} & \textit{@100\%}&&
    \textit{@70\%}  & \textit{@100\%} \\
    &  &&  &&  \textit{steps} && \textit{steps} && 
    \textit{steps} & \textit{steps} &&
    \textit{steps}  & \textit{steps} \\
    \midrule
    \bf World Knowledge & 4 && 15.90 && \underline{22.19} && 18.69 && 21.59 & \textbf{22.70} && 20.64 & 21.72 \\
    \bf Reading Comprehension & 4 && 46.30 && 53.00 && 51.00 &&  53.55 & \underline{54.55} && 54.35 & \textbf{54.93} \\
    \bf Commonsense Reasoning & 7 && 57.76 && 61.99 && 58.30 && 61.27 & 61.67 && \underline{62.00} & \textbf{62.10} \\
    \bf LAMBADA & 1 && 26.90 && 36.20 && 31.10 &&  \underline{50.70} & 48.30 && 48.00 & \textbf{53.00} \\
    \bf SuperGLUE & 8 && 61.59 && 65.53 && 62.91 && \textbf{66.30} & 65.28 && \underline{65.99} & 65.58 \\
    \bf NLG & 3 && 3.13 && 4.60 && 3.40 && 4.63 & 4.73 &&\underline{4.80} & \textbf{4.83} \\
    \bf MBPP & 1 && 9.60 && 16.20 && 11.40 && 15.60 & \underline{17.00} && 16.60 & \textbf{17.80} \\
    \midrule
    \bf Average & 28 && 42.56 && 47.32 && 44.39 &&  47.86 & \underline{47.94} && 47.89 & \textbf{48.26} \\
    \bottomrule
    \end{tabular}}
\end{table}

\subsection{Improved downstream performance realized via \salt{} post SFT}
\label{sec:exp-sft}

Tab.~\ref{tab:evals_all_grouped} already showcases that \salt{} can leverage SLMs to obtain large pre-trained LMs that out-perform widely adopted standard pre-training 
(\baseline{}).
That said, all the pre-trained models (including \baseline{}) exhibit relatively poor few-shot performance on certain benchmarks, e.g., NLG or summarization task (in Tab.~\ref{tab:evals_all_grouped})
and also MATH~\citep{hendrycks2021measuring} and ANLI~\citep{nie2020-adversarial} benchmarks. This raises the question of whether \salt{} is beneficial 
for downstream application performance.

Supervised fine-tuning (SFT) 
is a standard approach to convert a pre-trained LM into a domain-specific proficient model. 
We employ SFT on a range of downstream tasks covering three domains, namely arithmetic reasoning, NLG or summarization, and NLI. 
For each downstream benchmark, we perform SFT on the pre-trained LMs obtained via \baseline{}, \salt{}, and \saltds{}.
During SFT, we train each pre-trained LM for $10$K steps with Adafactor algorithm~\citep{shazeer2018adafactor}. We utilize cosine learning rate schedule \citep{loshchilov2017sgdr} with a peak learning rate of 1e-4 and linear warm-up phase consisting of 200 steps.
These hyperparameters were optimized for the \baseline{};~
and we did not conduct hyperparameter tuning for the pre-trained LMs resulting from \salt{} approaches. Finally, we employ greedy decoding during the evaluation of the fine-tuned LMs.
Tab.~\ref{tab:sft-results} presents SFT performance on six benchmarks covering the aforementioned three downstream domains. 
\salt{} consistently outperforms \baseline{} on all benchmarks and, in most cases, 
\saltds{} performs the best. 
This showcases that \salt{} (\textit{with} or \textit{without} data selection) enables significant improvements for a range of \textit{difficult} downstream domains even when the corresponding pre-trained LMs exhibits only a
small performance gains over \baseline{}.

\begin{table*}[t]
    \caption{\textbf{Supervised fine-tuning (SFT) results.}~Performance of various pre-trained checkpoints on downstream tasks after SFT. For each benchmark, pre-trained 2.8B models are fine-tuned on the corresponding train split and evaluated on the validation split (test split in case of GSM8K). \textit{Acc}, \textit{Rg1}, \textit{Rg2}, and \textit{RgL} represent the \textit{Accuracy}, \textit{Rouge-1}, \textit{Rouge-2}, and \textit{Rouge-Lsum} metrics, respectively.
    } \label{tab:sft-results}
    \centering
    \setlength{\tabcolsep}{2.5pt}
    \scalebox{0.96}{
    \begin{tabular}{@{}lccccccccccccccc@{}}
    \toprule
     & \multicolumn{1}{c}{\textbf{GSM8K}} && \multicolumn{3}{c}{\textbf{XSum}} && \multicolumn{3}{c}{\textbf{CNN/DailyMail}} && \multicolumn{1}{c}{\textbf{ANLI-R1}} && \multicolumn{1}{c}{\textbf{ANLI-R2}} && \multicolumn{1}{c}{\textbf{ANLI-R3}} \\
     \cmidrule{2-2} \cmidrule{4-6} \cmidrule{8-10}  \cmidrule{12-12} \cmidrule{14-14} \cmidrule{16-16}
     & 
     \textit{Acc} &&
     \textit{Rg1} & \textit{Rg2} & \textit{RgL} &&
     \textit{Rg1} & \textit{Rg2} & \textit{RgL} &&
     \textit{Acc} &&
     \textit{Acc} &&
     \textit{Acc} \\
     \midrule
     \baseline & 
     31.84 &&
     43.39 & 21.09 & 35.91 &&
     42.84 & 20.43 & 40.38 &&
     63.70 &&
     56.90 &&
     57.83 \\[1.5mm]
     \salt & 
     \underline{34.87} &&
     \underline{43.45} & \underline{21.21} & \underline{36.04} &&
     \underline{43.19} & \underline{20.65} & \underline{40.74} &&
     \underline{67.00} &&
     \bf 57.80 &&
     \bf 59.67 \\[1.5mm]
     \saltds & 
     \bf 35.25 &&
     \bf 43.77 & \bf 21.44 & \bf 36.24 &&
     \bf 43.41 & \bf 20.87 & \bf 40.95 &&
     \bf 67.30 &&
     \underline{57.70} &&
     \underline{59.58} \\
    \bottomrule
    \end{tabular}}
\end{table*}

\subsection{\slm{} enables fast learning on easy examples}
\label{sec:performance_slicing_on_example_hardness}
\begin{table}
  \caption{\textbf{Few-shot evaluation on different buckets of XLSum-EN.}~Each number shows average Rouge-2 scores on the corresponding bucket. We use \hlg{gray}, \hla{green}, and \hlr{red} to highlight the results similar to, better than, and worse than \baseline{} performance, respectively.
  }
  \label{tab:example_difficulty-xlsum}
  \centering
\setlength{\tabcolsep}{2pt}
\scalebox{0.96}{
\renewcommand{\arraystretch}{1.0}
\begin{tabular}[t]{@{}lp{2.5cm}ccc@{}} 
\toprule  
&  \centering \bf Evaluation stage (steps) & \bf Easy & \bf Medium & \bf Hard \\
\midrule 
\slm & \centering Final (208K) & 8.04 & 0.43 & 0.00 \\ 
\midrule
\baseline{} & \centering \multirow{3}{*}{Early (36K)} & \hlg{6.15} & \hlg{1.61} & \hlg{0.71} \\ 
\rkd{} & & \hla{6.76} & \hlr{1.40} & \hlr{0.58} \\
\salt{} & & \hla{6.76} & \hlr{1.40} & \hlr{0.58} \\
\midrule
\baseline{} & \centering \multirow{3}{*}{Final (208K)} & \hlg{8.80} & \hlg{2.52} & \hlg{0.97} \\ 
\rkd{} & & \hlr{7.87} & \hlr{1.68} & \hlr{0.74} \\ 
\salt{} & & \hla{9.68} & \hla{2.67} & \hla{0.99} \\ 
\midrule 
\end{tabular}}
\end{table}

Recall that \salt{} aims to realize quality and efficiency gains for LLM pre-training by quickly transferring the predictive distribution of an SLM to the LLM via KD, focusing on the `easy' regions of the data distribution where the SLM performs well. Subsequently, \salt{} falls back on ground truth next-token-based supervision to refine LLM's performance on the `hard' regions where the SLM fares poorly. Here, we set out to empirically demonstrate that this key intuition behind \salt{} is indeed borne out in practice. Focusing on various few-shot eval benchmarks, we partition instances in each benchmark into `easy', `medium', and `hard' buckets based on the teacher SLM's performance (see Appendix~\ref{appen:difficulty_slice} for details). We then evaluate \baseline{}, \salt{}, and \rkd{} pre-trained LLMs on these individual buckets after $n_{\rm KD} = 36$K training steps when the KD phase of \salt{} ends as well as at the end of the pre-training, i.e., after $208$K steps. Tab.~\ref{tab:example_difficulty-xlsum} presents these results on the XLSum-EN task (see Appendix~\ref{appen:difficulty_slice} for results on other benchmarks), which validate: (1) KD from \slm{} quickly enables LLM to perform well on `easy' instances; and (2) standard pre-training after KD phase ending at $n_{\rm KD}$-th step helps LLM performance on `hard' instances the most.

\section{Related work}
\label{sec:related}

Here, we provide a brief account of related work, focusing on the prior work on assisting large model training via small models, data selection, and theoretical treatment of KD. Due to the page limit, we defer the discussion on various KD methods for language modeling that are not pertinent to the main objective of this work to Appendix~\ref{appen:bg}.

\noindent\textbf{Aiding large model training with small models.}
Small models often help identify good hyper-parameters that can be utilized for large model training with minimal modifications~\citep{yang2021tuning}. Progressive or stage-wise training methods~\citep{gong2019efficient, gu2020transformer, reddi2023stacking} train a large model in stages where the parameters for the model at a given stage get initialized based on the parameters of a smaller model from the previous stage. Another related line of work simply informs the large model initialization based on a smaller model without resorting to progressive stage-wise training~\citep{trockman2023mimetic, wang2023lemon, samragh2024scaling}. Most of these works crucially depend on the architectural overlaps between the small and large models, e.g., requiring them to share same depth, width, or more generally same model family.
In contrast, one can distill from a smaller model to a larger model without such constraints. \citet{furlanello2018born} study self-distillation to iteratively train a model, with the final model from an iteration acting as a teacher for the next iteration.
More closer to the setting studied in this work, \textit{albeit} in image classification setting, \citet{yuan2020revisiting, xie2020self} consider distillation
from a weaker model.
In the work that is closest to our proposal, \citet{qin2022ki} 
distill an LM from a smaller LM.
Moreover, they also 
distill
during the early phase of larger student LM pre-training to avoid negative impact on large LM’s final performance. We demonstrate the utility of such an approach with larger LMs as well as larger pre-training corpus, with a wider range of evaluation benchmarks. Furthermore, we provide a statistical framework to rigorously justifies the utility of a seemingly weaker teacher during LLM pre-training. Other recent efforts on using smaller LMs to boost larger LMs have primarily focused on
fine-tuning~\citep{yang2024smalltolarge,mitchell2024an} or alignment~\citep{burns2023weak}, while 
we focus on pre-training which is the most compute and data intensive phase of LLM development. 

\textbf{Data Selection.} 
\cite{ankner2024perplexed} select examples with reference model sequence-level log-perplexities in a specified range. They aggregate per-token log-perplexities by taking the mean, whereas, we take the median and also zero out noisy tokens.
\cite{mindermann2022prioritized} select examples and \cite{lin2024rho1} select tokens for training based on \textit{excess} training loss over a reference model. While 
we perform \textit{offline} data selection before training,
these two works select data from training batches on the fly.
Methods which encourage diversity through semantic embeddings (e.g., \cite{abbas2023semdedup,tirumala2024d4}) can yield a better final model when employed in \saltds{}.
Please refer to \cite{albalak2024asurvey} for a comprehensive survey of data selection techniques for LMs.

\noindent\textbf{Theoretical understanding of knowledge distillation.}~Focusing on (deep) linear networks, \citet{phuong19understanding} study the generalization bounds for KD and relate its success to various factors including data geometry and optimization bias. \citet{menon21statistical,ren2022better} show that KD leads to reduced variance of the training objective, thereby improving generalization, while also relating the effectiveness of KD to teacher's ability to approximate Bayes class probabilities. \citet{cotter2021distilling} argue that KD can improve 
generalization as long as teacher approximates a suitable transformation of Bayes class probabilities. \cite{dao2021knowledge} study KD from the perspective of semiparameteric inference to analyze the excess risk of distilled student. Focusing on self-distillation for kernel ridge regression, \cite{mobahi2020self} show that distillation can enhance regularization. \cite{allen-zhu2023towards} explain the utility of distillation via better feature learning. \citet{xu2020towards} study the interplay between optimization and 
label smoothing via teacher-provided soft labels.
Focusing on linearized neural networks, \citet{harutyunyan2023supervision} attribute the success of KD to the reduced supervision complexity. 
More recently, \citet{safaryan2023variance} argue that KD can act as a form of partial variance reduction, thereby improving convergence. We would like to highlight that the existing literature \textit{does not} provide
generalization bounds for
KD in a \textit{sequence} learning setting such as language modeling, and we provide the first statistical treatment of KD for language modeling.
Notably, similar to our work, \citet{xu2020towards, nagarajan2023student} also explore the utility of KD only during an early-phase of student training, albeit not in a language modeling setting.

\section{Conclusion}\label{sec:conclusion}
We conduct a systematic study of the utility of SLMs to improve both training efficiency and performance of LLM pre-training. 
Towards this, we introduce a novel statistical framework for KD in the context of language modeling, which guides the design of \salt{} -- a simple KD-based approach that selectively transfers predictive distribution from an SLM to an LLM. 
We further enhance \salt{} 
and perform explicit data selection via SLMs
to effectively transfer knowledge from SLMs to LLMs.
\salt{} significantly reduces the pre-training time for LLMs while ensuring good overall quality as measured by the LLM's few-shot performance as well as downstream performance after fine-tuning.

While \salt{} can play a crucial role in efficiently sustaining the trend of developing LLMs with increasing capabilities, it also has the potential to help institutions train proficient lightweight LMs. Conventionally, one distills a large powerful model into a lightweight model with good quality. However, given that many LLMs are proprietary, such strong-to-weak distillation is not feasible at many institutions. Our proposed approach can leverage even \textit{smaller} LMs to enhance the quality of a small LM with acceptable inference cost.
Exploring if our proposed approach can introduce new capabilities in a general-purpose small LM with the help of one or multiple \textit{smaller} LMs that are experts in their respective domains is an interesting avenue for future work. Interestingly, our data selection approach in \saltds{} demonstrates that it is indeed possible to leverage data selection to improve KD for language modeling. Building on this initial investigation and further exploring and extending data selection approaches tailored to \saltds{} is another interesting line of future work.

\section*{Acknowledgments}

The authors would like to thank Ke Ye, Julian Odell, and Rina Panigrahy for their valuable input during the early stages of this research effort. 

\bibliographystyle{plainnat}
\bibliography{main}

\clearpage
\newpage

\begin{center}
\textcolor{black}{\Large{}Appendix}{\Large\par}
\par\end{center}

\appendix

\section{Language modeling and knowledge distillation}
\label{appen:bg}

\subsection{Language modeling via Transformer-based models}
\label{appen:lm-transformers}

Here, we briefly discuss how Transformers are typically employed for language modeling in modern systems. 
Given a context $\rvx_{\leq t} = [x_1,\ldots, x_{t}] \in \VC^{t}$, a Transformer-based LM first produces a sequence of $d$-dimensional token embeddings 
$E(\rvx_{\leq t}) = [\re_{x_1}^{\top}, \re_{x_2}^{\top},\ldots, \re_{x_{t}}^{\top}] \in \R^{d\times t}$, where $\rve_{v} \in \R^d$ denotes the token embedding for 
$v \in \VC$. A Transformer network $f_{\bpsi}$ then processes 
$E(\rvx_{\leq t})$ to produce a target embedding $f_{\bpsi}\big(E(\rvx_{\leq t})\big) \in \R^{d}$, 
which is 
multiplied by $W \in \mathbb{R}^{V \times d}$, namely a classification layer, to obtain a 
logit vector $u_{\rvx_{\leq t}}:= W f_{\bpsi}\big(E(\rvx_{\leq t})\big) \in \R^{V}$. Accordingly, we have $\btheta = \{E, \bpsi, W\}$ as the parameters of the LM. Applying softmax operation on the logit vector produces the probability that the LM assigns to each token in $\VC$ as the possible continuation (also known as next token) to the context $\rvx_{\leq t}$:
\begin{align}
  P_{\btheta}(v | \rvx_{\leq t}) = \frac{\exp({u_{\rvx_{\leq t}}(v)} / \tau )}{\sum_{v' \in \VC} \exp({u_{\rvx_{\leq t}}(v')} /\tau )}, \qquad \forall v \in \VC.  
\end{align}
Here, $\tau$ denotes the (inverse) temperature associated with the softmax operation. Unless stated otherwise, we assume that $\tau = 1$.

\subsection{Other common variants of knowledge distillation for LM}
\label{append:kd-variants}
\noindent\textbf{Top-$k$ token-level KD.}~Instead of aligning the teacher and student's full per-token prediction distributions, one could only match 
these distribution on $\mathcal{T} \subset \mathcal{V}$, e.g., $k \ll V$ elements of $\mathcal{V}$ that receive the highest scores from the teacher:
\begin{align}
\label{eq:tkd_loss_ce_subset}
\ell(\mathbf{x}; \bzeta \to \btheta) = - \sum_{t = 1}^{T} \Big( \sum_{v \in \TC} P^{\TC}_{\bzeta}(v | \rvx_{\leq t -1}) \cdot  \log P^{\TC}_{\btheta}(v | \rvx_{\leq t -1})\big) \Big),
\end{align}
where $P^{\mathcal{T}}$ denotes the restriction of $P$ (defined over $\mathcal{V}$) to $\mathcal{T}$:
\begin{align}
    P^{\mathcal{T}}(v) = \begin{cases}
    \frac{P(v)}{\sum_{v' \in \mathcal{T}}P(v')} & \text{if}~v \in \mathcal{T}, \\
    0 & \text{otherwise}.
    \end{cases}
\end{align}

\noindent\textbf{Sequence-level KD.}~Unlike token-level KD, sequence-level KD aims to align teacher and student's distributions on all sequences up to sequence length $T$. In particular, the sequence-level KD loss takes the form:
\begin{align}
\label{eq:seq_loss_ce}
\ell^{\rm seq}(\mathbf{x}; \bzeta \to \btheta) = - \sum_{\tilde{\rvx} \in \VC^{\leq n}}  P_{\bzeta}(\tilde{\rvx}) \cdot \log P_{\btheta}(\tilde{\rvx})
\end{align}
In practice, it's natural to focus on a subset of all candidate target sequences $\UC \subset \mathcal{V}^{\leq T}$:
\begin{align*}
\ell^{\rm seq}(\mathbf{x}; \bzeta \to \btheta) = - \sum_{\tilde{\rvx} \in \UC}  P_{\bzeta}(\tilde{\rvx}) \cdot \log P_{\btheta}(\tilde{\rvx})
\end{align*}
A common choice for $\UC$ is the set of say $k$ most likely sequences under the teacher's distribution $P_{\bzeta}$.

\subsection{Recent literature on knowledge distillation (KD) for language modeling}
A large body of literature focuses on utilizing KD~\citep{Bucilua:2006,Hinton:2015} as a core technique to improve LMs~\citep{kim2016sequence, gou2021knowledge, xu2024survey}. For instance, \citet{sanh2019distilbert,turc2019well,wang2020minilm,sun2019patient,jiao2019tinybert} relied on KD to compress BERT-style LMs during pre-training, fine-tuning, or both. More recently, KD has been primarily employed in the instruction-tuning or fine-tuning phase where a general purpose LM is adapted to a specific collection of tasks~\citep{xu2024survey}. Black-box KD methods for LM only assume access to training sequences sampled from a teacher LM
~\citep{rohan2023alpaca, fu2023specializing, peng2023instruction}.
With access to token-level distributions from teacher LM, \textit{token-level distillation} from teacher LM to student LM is possible~\citep{kim2016sequence}. In contrast, \textit{sequence-level distillation} involves sampling training sequences from the teacher LM, the student LM, or both before aligning teacher and student’s predictive distribution on such sequences~\citep{kim2016sequence, agarwal2024onpolicy, gu2024minillm, wen2023fdivergence}.

\section{Deferred proofs from Section~\ref{sec:theory}}
\label{appen:theory}

\subsection{Proof of Theorem~\ref{thm:esr}}
\label{appen:esr}

Before stating the formal version of Theorem~\ref{thm:esr} and its proof, let us recall the necessary notation. Given a function class for student LMs $\Theta$, $\thetaest$ denotes the LM obtained by minimizing the training objective for KD in \eqref{eq:erm_distill}, i.e., 
\begin{align}
\label{eq:kd-erm-theta-appen}
\thetaest := 
\argmin_{\btheta \in \Theta} \RCENW(\btheta).
\end{align}
Further, $\btheta^{*}$ represents the optimal or best performing LM in $\Theta$, i.e.,
\begin{align}
\label{eq:optimal-theta-appen}
\btheta^{*} = \argmin_{\btheta \in \Theta} \RCE(\btheta).
\end{align}
Finally, recall our assumption regarding the bounded loss values.
\begin{assumption}
\label{assm:loss_bound-appen}
Given a function class $\Theta$ for (student) LM, the per-token log-loss with at most $T$-token long sequences for the underlying function class $\Theta$ is bounded by $M$, i.e.,
\begin{align}
\sup_{\btheta \in \Theta; \rvx \in \VC^{\leq T-1}} \max_{v \in \VC}\vert \log P_{\btheta}(v| \rvx) \vert \leq M.
\end{align}
\end{assumption}

Towards establishing Theorem~\ref{thm:esr}, we first state the following intermediate result.
\begin{proposition}
\label{prop:esr}
Let $\thetaest$ and $\btheta^{*}$ be as defined in \eqref{eq:kd-erm-theta-appen} and \eqref{eq:optimal-theta-appen}, respectively. Then, under Assumption~\ref{assm:loss_bound-appen}, the excess surrogate risk for $\thetaest$ satisfies the following.
\begin{align}
R(\thetaest ) - R(\btheta^{*}) &\leq 
\frac{4M \omega}{T} \cdot \sum_{t=1}^{T} \mathbb{E}_{\rvx_{\leq t-1} \sim \DCal}\mathsf{D}_{\rm TV}\Big(P_{\bzeta,\rho}(\cdot|\rvx_{\leq t-1}), \DCal(\cdot|\rvx_{\leq t-1})\Big) 
\nonumber \\
&\qquad 
+ \left(\RCEW(\thetaest )  -   \RCENW(\thetaest ) \right) \; + 
\left(\RCENW(\btheta^{*}) -   \RCEW(\btheta^{*})\right),
\end{align}
where
$\mathsf{D}_{\rm TV}(\cdot, \cdot)$ denotes the total-variation distance between two probability distributions.
\end{proposition}

\begin{proof}[Proof of Proposition~\ref{prop:esr}]
For convenience, recall that
\begin{align*}
\RCEW(\btheta)
&=\mathbb{E}_{\rvx} \Big[(1-\omega) \cdot \ell(\rvx; \btheta) + \omega \cdot \ell (\rvx; \bzeta^\rho \rightarrow \btheta)\Big]\\
&= \mathbb{E}_{\rvx \sim \DCal}\left[\frac{1}{T}\sum_{t=1}^{T} \mathsf{CE}\big( P^{(x_{t}, \omega)}_{\zeta, \rho},  P_{\btheta}(\cdot | \rvx_{\leq t-1})\big)\right] \nonumber \\
& = \frac{1}{T}\sum_{t=1}^{T} \mathbb{E}_{\rvx_{\leq t-1} \sim \DCal}\big[\mathsf{CE}\big(\underbrace{(1-\omega)\cdot\DCal(\cdot|\rvx_{\leq t-1}) + \omega \cdot P_{\bzeta,\rho}(\cdot|\rvx_{\leq t-1})}_{{P}^{(\DCal, \omega)}_{\zeta, \rho}(\cdot | \rvx_{\leq t-1})}, P_{\btheta}(\cdot | \rvx_{\leq t-1})\big)\big] \nonumber \\
& = \frac{1}{T}\sum_{t=1}^{T} \mathbb{E}_{\rvx_{\leq t-1} \sim \DCal}\big[\mathsf{CE}\big({P}^{(\DCal, \omega)}_{\zeta, \rho}(\cdot | \rvx_{\leq t-1}), P_{\btheta}(\cdot | \rvx_{\leq t-1})\big)\big].
\end{align*}

Note that we have
\begin{align}
\label{eq:proof-esr-1}
&\RCE(\thetaest ) - \RCE(\btheta^{*}) \nonumber \\
&\overset{(i)}{=} \RCE(\thetaest ) - \RCE(\btheta^{*}) - \Big(\RCEW(\thetaest ) - \RCEW(\btheta^{*})\Big) + \Big(\RCEW(\thetaest ) - \RCEW(\btheta^{*}) \Big) \nonumber \\
&\overset{(ii)}{=}\frac{1}{T}\sum_{t=0}^{T-1} \mathbb{E}_{\rvx_{1:t} \sim \DCal}\Big[\sum_{v \in \VC}\Big({P}^{(\DCal, \omega)}_{\zeta, \rho}(v | \rvx_{1:t}) - \DCal(v|\rvx_{1:t}) \Big)\cdot\Big(\log P_{\thetaest}(v | \rvx_{1:t}) - \log P_{\btheta^{*}}(v | \rvx_{1:t})\Big) \Big] + \nonumber \\
& \qquad \RCEW(\thetaest ) - \RCEW(\btheta^{*}) \nonumber \\
&\overset{(iii)}{\leq}\frac{1}{T}\sum_{t=0}^{T-1} \mathbb{E}_{\rvx_{1:t} \sim \DCal}\Big[\Big\|{P}^{(\DCal, \omega)}_{\zeta, \rho}(\cdot | \rvx_{1:t}) - \DCal(\cdot|\rvx_{1:t}) \Big\|_1\cdot\Big\|\log P_{\thetaest}(\cdot | \rvx_{1:t}) - \log P_{\btheta^{*}}(\cdot | \rvx_{1:t})\Big\|_{\infty} \Big] + \nonumber \\
& \qquad  {\RCEW(\thetaest ) - \RCEW(\btheta^{*})} \nonumber \\
&\overset{(iv)}{\leq} \frac{4M \omega}{T} \cdot \sum_{t=0}^{T-1} \mathbb{E}_{\rvx_{1:t} \sim \DCal}\Big[\mathsf{D}_{\rm TV}\Big(P_{\bzeta,\rho}(\cdot|\rvx_{1:t}), \DCal(\cdot|\rvx_{1:t})\Big) \Big] +  \underbrace{\RCEW(\thetaest ) - \RCEW(\btheta^{*})}_{{\rm (I)}},
\end{align}
where 
$(i)$ follows from adding and subtracting $\RCEW(\thetaest ) - \RCEW(\btheta^{*})$; 
$(ii)$ employs
the definition of $\RCE(\thetaest ),~\RCE(\btheta^{*}),~\RCEW(\thetaest)$, and $\RCEW(\btheta^{*})$; 
$(iii)$ invokes H\"{o}lder's inequality; 
and $(iv)$ follows from the definition of total-variation distance $\mathsf{D}_{\rm TV}(\cdot, \cdot)$ and the fact that underlying per-token loss terms are bounded by $M$.

Next, we focus on the term {\rm (I)} in \eqref{eq:proof-esr-1}:
\begin{align}
\label{eq:proof-esr-2}
\RCEW(\thetaest ) - \RCEW(\btheta^{*})
&\overset{(i)}{=} \RCEW(\thetaest ) - \RCENW(\thetaest) + \RCENW(\thetaest ) - \RCENW(\btheta^{*}) + \RCENW(\btheta^{*})  - \RCEW(\btheta^{*}) \nonumber \\
&= \Big(\RCEW(\thetaest )  -   \RCENW(\thetaest ) \Big) + \Big(\RCENW(\btheta^{*}) -   \RCEW(\btheta^{*}) \Big) \; + \RCENW(\thetaest) - \RCENW(\btheta^{*}) \nonumber \\
&\overset{(ii)}{\leq} \Big(\RCEW(\thetaest )  -   \RCENW(\thetaest ) \Big) + \Big(\RCENW(\btheta^{*}) -   \RCEW(\btheta^{*}) \Big)
\end{align}
where $(i)$ follows by adding and subtracting $\RCENW(\thetaest)$ and $\RCENW(\btheta^{*})$; and $(ii)$ holds as $\thetaest$ is the minimizer of $\RCENW(\cdot)$ in $\Theta$ which implies that $\RCENW(\thetaest) - \RCENW(\btheta^{*}) \le 0$.
\eqref{eq:proof-esr-1} and \eqref{eq:proof-esr-2}.
\end{proof}

Note that the bound on excess surrogate risk in Proposition~\ref{prop:esr} decomposes into three terms: 
\begin{itemize}[leftmargin=8mm, itemsep=1mm, partopsep=0pt,parsep=0pt]
\item First term captures the \textit{divergence} between the ground truth per-token distribution and the teacher-induced per-token distribution leveraged during KD; and
\item The last two terms corresponds to the deviation between empirical and population surrogate risks for the empirical risk minimizer $\thetaest$ and population risk minimzer ${\btheta}^{*}$ within the function class $\Theta$. Note that since, $\thetaest$ is a a random variable in itself (which depends on the training sample $\SCal_N$), one typically needs to bound the deviation uniformly over all functions $\btheta \in \Theta$. As we will see next, one can bound these deviations in terms of the properties of both model class $\Theta$ as well as the teacher-induced per-token distributions.
\end{itemize}

In order to make the excess surrogate risk bound in Proposition~\ref{prop:esr} explicit, we need to bound the third term via a computable quantity. We apply sample variance-based bounds from \citet{maurer2009empirical} to get the following result.

\begin{theorem}[Formal version of Theorem~\ref{thm:esr}]
\label{thm:esr-appen}
Suppose Assumption~\ref{assm:loss_bound-appen} holds. Let $\FC^{\bzeta,\rho,\omega}$ be a function class that maps elements in $\VC^{T}$ to $[0, M]$ as defined below:
\begin{align}
\label{eq:FC_def}
    \FC^{\omega} := \FC^{\bzeta,\rho,\omega} \triangleq \left\{ \rvx \mapsto \frac{1}{T}\sum_{t = 1}^{T} \mathsf{CE}\big({P}^{(\DCal, \omega)}_{\zeta, \rho}(\cdot | \rvx_{\leq t - 1}), P_{\btheta}(\cdot | \rvx_{\leq t-1})\big),~\forall \rvx \in \VC^T, \btheta \in \Theta\right\}.
\end{align}
For $\epsilon > 0$, let $\mathcal{N}_{\infty}(\epsilon, \FC^{\bzeta,\rho,\omega}, N)$ denote the growth function for the function class $\FC^{\bzeta,\rho,\omega}$, i.e., 
\begin{align}
    \mathcal{N}_{\infty}(\epsilon, \FC^{\bzeta,\rho,\omega}, N) \triangleq \sup_{\rmX = (\rvx^{(1)},\ldots, \rvx^{(N)}) \in \VC^{T \times N}} \mathcal{N}(\epsilon, \FC^{\bzeta,\rho,\omega}(\rmX), \|\cdot\|_{\infty}),
\end{align}
where $\mathcal{N}(\epsilon, \FC^{\bzeta,\rho,\omega}(\rmX), \|\cdot\|_{\infty})$ denotes the smallest $\epsilon$-cover of the set
\begin{align*}
\FC^{\bzeta,\rho,\omega}(\rmX) = \left\{ \big(f(\rvx^{(1)}), f(\rvx^{(2)}),\ldots, f(\rvx^{(N)}) \big) : f \in \FC^{\bzeta,\rho,\omega}\right\} \subseteq \R^N
\end{align*}
with respect to $\|\cdot\|_{\infty}$ norm. Then, with probability at least $1 - \delta$, for all $\btheta \in \Theta$, we have
\begin{align}
 &\RCE(\thetaest ) - \RCE(\btheta^{*})
 \leq 
 \frac{4M \omega}{T} \cdot \sum_{t=1}^{T} \mathbb{E}_{\rvx_{\leq t-1} \sim \DCal}\mathsf{D}_{\rm TV}\Big(P_{\bzeta,\rho}(\cdot|\rvx_{\leq t-1}), \DCal(\cdot|\rvx_{\leq t-1})\Big) \nonumber \\
 &\qquad + \sqrt{\frac{18 V_N(f^{\thetaest},\SCal_N)\log\left(\frac{2\MC(N)}{\delta}\right)}{N}} + \frac{15 M \log\left(\frac{2\MC(N)}{\delta}\right)}{N-1} \nonumber \\
 & \qquad + \sqrt{\frac{2 V_N(f^{\btheta^{*}},\SCal_N)\log\left(\frac{4}{\delta}\right)}{N}} + \frac{7 M \log\left(\frac{4}{\delta}\right)}{3(N-1)},
\end{align}
where $\MC(N) \triangleq 10\cdot \mathcal{N}_{\infty}(1/N, \FC^{\bzeta,\rho,\omega}, 2N)$; $f^{\btheta}$ denotes
the function in $\FC^{\bzeta, \rho, \omega}$ that corresponds to $\btheta$, as per \eqref{eq:FC_def}; and $V_N(f^{\btheta},\SCal_N)$ denotes the sample variance
\begin{align}
    V_N(f^{\btheta},\SCal_N) = \frac{1}{N(N-1)}\sum_{1 \leq i < j \leq N} \big(f^{\btheta}(\rvx^{(i)}) - f^{\btheta}(\rvx^{(j)})\big)^2.
\end{align}
\end{theorem}

\begin{proof}[Proof of Theorem~\ref{thm:esr-appen}]
As discussed earlier, 
in light of Proposition~\ref{prop:esr}, we only need to bound two terms $\RCEW(\thetaest )  -   \RCENW(\thetaest )$ and $\RCENW(\btheta^{*}) -   \RCEW(\btheta^{*})$ to obtain the desired result. Now utilizing Theorem~6 and Theorem~4 (with $\delta$ replaced with $\delta/2$) in \citet{maurer2009empirical} to bound the two terms, respectively, completes the proof of Theorem~\ref{thm:esr-appen}.
\end{proof}

\subsection{Token-level generalization bound}
\label{sec:token_level_bound_pf}
Before providing a proof of Theorem~\ref{thm:token-level-generalization-bound}, we first introduce some intermediate results that are needed to prove the theorem. Recall that our training sample $\SCal_N = \{\rvx^{(i)} = [x^{(i)}_1,\ldots, x^{(i)}_T]\}_{i \in [N]}$ comprises $N$ independent sequences such that $\rvx^{(i)} \sim \DCal, \forall i \in [N]$. 
With $\ell^\omega(\rvx^{(i)}; \btheta)$ representing the KD loss on $i$-th sequence, we define
the random variables
\begin{align}
    Z^{(i)}_0 &= \mathbb{E}[\ell^\omega(\rvx^{(i)}; \btheta)], \nonumber \\ 
    Z^{(i)}_t &= \mathbb{E}\left[\ell^\omega(\rvx^{(i)}; \btheta) \; \vert \; \rvx^{(i)}_{\leq t} \right],~\text{for}~1 \leq t \leq T, 
    \label{eq:Z_def}
\end{align}
where $Z^{(i)}_T = \mathbb{E}\left[\ell^\omega(\rvx^{(i)}; \btheta) \; \vert \; \rvx^{(i)} \right] = \ell^\omega(\rvx^{(i)}; \btheta)$. Note that $\{Z^{(i)}_t\}_{0 \leq t \leq T}$ is a \textit{Doob martingale sequence} with respect to the natural filtration $\{\mathcal{F}^{(i)}_t\}_{0 \leq t \leq T}$ of the random variables $\{x^{(i)}_1,\ldots, x^{(i)}_{t}\}$~\citep[pg 297]{Ross:1983}. Accordingly, we define a \textit{martingale difference sequence} $\{\xi^{(i)}_{t}, \mathcal{F}^{(i)}_t\}_{t \in [T]}$ such that for $t \in [T]$,
\begin{align}
\xi^{(i)}_{t}:= \xi_{t}(\rvx^{(i)}; \btheta) = Z^{(i)}_{t-1} - Z^{(i)}_{t} = \mathbb{E}\left[\ell^\omega(\rvx; \btheta) | \rvx^{(i)}_{\leq t - 1} \right] - \mathbb{E}\left[\ell^\omega(\rvx; \btheta) | \rvx^{(i)}_{\leq t} \right].
\end{align}
As per Assumption~\ref{assm:xi}, the following holds for each $t \in [T]$:
\begin{align}
|\xi^{(i)}_{t}(\rvx; \btheta)| & \leq C_t \leq C, \label{eq:assm_xi1_appen} \\
\mathbb{E}\left[\left(\xi^{(i)}_{t}\right)^2 | \rvx_{\leq t-1}\right] &\leq V_t.
\end{align}

We are ready to state the first intermediate result which bounds the moment generating function for the following random variable associated with the KD loss on the $i$-th training sequence:
$$
Z^{(i)}_{0} - Z^{(i)}_{T} =
\mathbb{E}\left[\ell^\omega(\rvx^{(i)}; \btheta)\right] - \ell^\omega(\rvx^{(i)}; \btheta).
$$


\begin{lemma}
\label{lem:mgf_u}
Under Assumption~\ref{assm:xi}, the following holds for each $i \in [N]$:
\begin{align}
\mathbb{E}\left[e^{\lambda\cdot(Z^{(i)}_{0} - Z^{(i)}_{T})/C}\right] \leq \exp\left(T\cdot f\Big(\lambda,\frac{1}{T}\sum_{t=1}^T\frac{V_t}{C^2}\Big)\right),
\end{align}
where, for $\lambda \geq 0$ and $s \geq 0$,
\begin{align}
f(\lambda, s) \triangleq \log \left(\frac{1}{1+s} \cdot \exp(-\lambda s) + \frac{s}{1+s} \cdot \exp(\lambda)\right).
\end{align}
\end{lemma}

\begin{proof}
Note that 
\begin{align}
\label{eq:mgf_u_proof1}
\mathbb{E}\left[e^{\lambda\cdot(Z^{(i)}_{0} - Z^{(i)}_{T})/C}\right] 
&= \mathbb{E}\left[e^{\lambda\cdot\sum_{t=1}^T\xi^{(i)}_{t}/C}\right] \nonumber \\
&= \mathbb{E}\left[\mathbb{E}\Big[e^{\lambda\cdot\sum_{t=1}^T\xi^{(i)}_{t}/C}\vert \rvx^{(i)}_{\leq T-1}\Big]\right] \nonumber \\
&\overset{(i)}{=}  \mathbb{E}\left[e^{\lambda\cdot\sum_{t=1}^{T-1}\xi^{(i)}_{t}/C}\cdot\mathbb{E}\Big[e^{\lambda\cdot\xi^{(i)}_{T}/C}\vert \rvx^{(i)}_{\leq T-1}\Big]\right] \nonumber \\
& \overset{(ii)}{\leq}  \mathbb{E}\left[e^{\lambda\cdot\sum_{t=1}^{T-1}\xi^{(i)}_{t}/C}\cdot e^{f\big(\lambda,\;  \frac{1}{C^2}\cdot\mathbb{E}\left[\left(\xi^{(i)}_{T}\right)^2 \vert \rvx^{(i)}_{\leq T-1}\right]\big)}\right] \nonumber \\
&\overset{(iii)}{\leq}  \mathbb{E}\left[e^{\lambda\cdot\sum_{t=1}^{T-1}\xi^{(i)}_{t}/C}\cdot e^{f(\lambda,\;  \frac{V_T}{C^2})}\right] \nonumber \\
& =  \mathbb{E}\left[e^{\lambda\cdot\sum_{t=1}^{T-1}\xi^{(i)}_{t}/C} \right]\cdot e^{f(\lambda,\;  \frac{V_T}{C^2})}
\end{align}
where $(i)$ follows as $e^{\lambda\cdot\sum_{t=1}^{T-1}\xi^{(i)}_{t}}$ is $\mathcal{F}^{(i)}_{T-1}$-measurable; $(ii)$ follows from \citep[Lemma 3.1]{fan2012hoeffding}; and $(iii)$ follows from Assumption~\ref{assm:xi} and the fact that, for $\lambda > 0$ and $s \geq 0$, $f(\lambda, s)$ is an increasing function in its second argument~\citep[Lemma~3.2]{fan2012hoeffding}.
By following the similar steps in \eqref{eq:mgf_u_proof1} for $\xi_{i, T-1}, \xi_{i, T-2},\ldots, \xi_{i, 1}$, we obtain that
\begin{align}
\label{eq:mgf_u_proof2}
\mathbb{E}\left[e^{\lambda\cdot(Z^{(i)}_{0} - Z^{(i)}_{T})/C}\right] \leq e^{\sum_{t=1}^T f(\lambda, \frac{V_t}{C^2})}.
\end{align}
According to \citep[Lemma~3.2]{fan2012hoeffding} that, for $\lambda \geq 0$ and $s \geq 0$, $f(\lambda, s)$ is a concave function in its second argument. Thus, it follows from Jensen's inquality that
\begin{align}
\label{eq:mgf_u_proof3}
\frac{1}{T}\sum_{t=1}^{T}f\left( \lambda, \frac{V_t}{C^2} \right) \leq f\left(\lambda, \frac{1}{T}\sum_{t=1}^{T}\frac{V_t}{C^2}\right).
\end{align}
By combining \eqref{eq:mgf_u_proof2} and \eqref{eq:mgf_u_proof3}, we have
\begin{align}
\label{eq:mgf_u_proof4}
\mathbb{E}\left[e^{\lambda\cdot(Z^{(i)}_{0} - Z^{(i)}_{T})/C}\right] \leq e^{T\cdot f\big(\lambda, \frac{1}{T}\sum_{t=1}^{T}\frac{V_t}{C^2}\big)},
\end{align}
which completes the proof.
\end{proof}


Now we can leverage Lemma~\ref{lem:mgf_u} to obtain the following concentration inequality for the KD training objective.


\begin{lemma}
\label{lem:U_concentration}
Let $\bzeta$ and $\btheta \in \Theta$ denote the teacher and student LM, respectively. Then, for $\epsilon > 0$, the following holds under Assumption~\ref{assm:xi}.
\begin{align}
\label{eq:KD_loss_concentration}
\mathbb{P}\left(\sum_{i=1}^N \left( \mathbb{E}\big[\ell^\omega(\rvx^{(i)}; \btheta)\big] - \ell^\omega(\rvx^{(i)}; \btheta)\right)/C \geq N\epsilon \right) \leq \exp\left(-\frac{N\epsilon^2}{2(\sum_{t}\frac{V_t}{C^2} + \frac{1}{3}\epsilon)}\right).
\end{align}
\end{lemma}

\begin{proof}
Recall that, as per our notation, we have
$$
\mathbb{E}\left[\ell^\omega(\rvx^{(i)}; \btheta)\right] - \ell^\omega(\rvx^{(i)}; \btheta) = Z^{(i)}_{0} - Z^{(i)}_{T}.
$$
Thus, 
\begin{align}
\label{eq:U_concentration_proof0}
\mathbb{P}\left(\sum_{i=1}^N \left( \mathbb{E}\big[\ell^\omega(\rvx^{(i)}; \btheta)\big] - \ell^\omega(\rvx^{(i)}; \btheta)\right)/C \geq N\epsilon \right) = \mathbb{P}\left(\sum_{i=1}^N \big( Z^{(i)}_{0} - Z^{(i)}_{T} \big)/C \geq N\epsilon \right). 
\end{align}
It follows from Markov's inequality that, for $\lambda \geq 0$,
\begin{align}
\label{eq:U_concentration_proof1}
\mathbb{P}\left(\sum_{i=1}^N \big( Z^{(i)}_{0} - Z^{(i)}_{T} \big)/C \geq N\epsilon \right)  &= \mathbb{P}\left(e^{\lambda \cdot \sum_{i=1}^N \big( Z^{(i)}_{0} - Z^{(i)}_{T} \big)/C} \geq e^{N\lambda\epsilon} \right) \nonumber \\
&\leq \frac{\mathbb{E}\left[e^{\lambda\cdot\sum_{i=1}^N \big( Z^{(i)}_{0} - Z^{(i)}_{T} \big)/C}\right]}{e^{N\lambda\epsilon}} \nonumber \\
&\overset{(i)}{=} \frac{\prod_{i \in [N]}\mathbb{E}\left[e^{\lambda \cdot \big( Z^{(i)}_{0} - Z^{(i)}_{T} \big)/C}\right]}{e^{N\lambda \epsilon}},
\end{align}
where $(i)$ follows as $\{Z^{(i)}_{0} - Z^{(i)}_{T} \big\}_{i \in [N]}$ are independent random variables. 
By combining \eqref{eq:U_concentration_proof1} with Lemma~\ref{lem:mgf_u}, we obtain that
\begin{align}
\label{eq:U_concentration_proof2}
\mathbb{P}\left(\sum_{i=1}^N \big( Z^{(i)}_{0} - Z^{(i)}_{T} \big)/C \geq N\epsilon \right)  &\leq e^{-N\cdot\big(\lambda\epsilon - T\cdot f(\lambda,\frac{1}{T}\sum_{t=1}^T\frac{V_t}{C^2})\big)}.
\end{align}
Since \eqref{eq:U_concentration_proof2} holds for each $\lambda \geq 0$, we have 
\begin{align}
\label{eq:U_concentration_proof3}
\mathbb{P}\left(\sum_{i=1}^N \big( Z^{(i)}_{0} - Z^{(i)}_{T} \big)/C \geq N\epsilon \right)  &\leq \inf_{\lambda \geq 0}e^{-N\cdot\big(\lambda\epsilon - T\cdot f(\lambda,\frac{1}{T}\sum_{t=1}^T\frac{V_t}{C^2})\big)}
\end{align}
Now as argued in the Proof of Remark 2.1 in \citet{fan2012hoeffding}, for $0 \leq \lambda < 3, s \geq 0$, we have
\begin{align}
\label{eq:f_lambda_s_bound}
f(\lambda,s) \leq (e^{\lambda} - 1 - \lambda)s \leq \frac{\lambda^2s}{2(1 - \frac{1}{3}\lambda)}.
\end{align} 
Thus, it follows from \eqref{eq:U_concentration_proof3} that
\begin{align}
\label{eq:U_concentration_proof4}
\mathbb{P}\left(\sum_{i=1}^N \big( Z^{(i)}_{0} - Z^{(i)}_{T} \big)/C \geq N\epsilon \right) &\leq \inf_{0 \leq \lambda < 3}\exp\left({-N\cdot\Big(\lambda\epsilon - \frac{\lambda^2}{2(1 - \frac{1}{3}\lambda)}\cdot \sum_{t}\frac{V_t}{C^2}\Big)}\right) \nonumber \\
&\leq \exp\left(-\frac{N\epsilon^2}{2(\sum_{t}\frac{V_t}{C^2} + \frac{1}{3}\epsilon)}\right).
\end{align}
This completes the proof.
\end{proof}


Equipped with Lemma~\ref{lem:U_concentration}, we are now ready to prove Theorem~\ref{thm:token-level-generalization-bound} below.


\begin{proof}[Proof of Theorem~\ref{thm:token-level-generalization-bound}]
Note that
\begin{align}
\label{eq:decomp_erm}
&\RCE(\thetaest ) - \RCENW(\thetaest) = 
\underbrace{\RCE(\thetaest ) -  \RCEW(\thetaest )}_{({\rm I})} \; + \;  \underbrace{\RCEW(\thetaest ) - \RCENW(\thetaest)}_{({\rm II})}.
\end{align}
Following the similar analysis used in the proof of Theorem~\ref{thm:esr}, we can bound the term ({\rm I}) to obtain the following.
\begin{align}
\label{eq:div_bound_erm}
\RCE(\thetaest ) -  \RCEW(\thetaest ) \leq \frac{4M \omega}{T} \cdot \sum_{t=1}^{T} \mathbb{E}_{\rvx_{\leq t-1} \sim \DCal}\mathsf{D}_{\rm TV}\Big(P_{\bzeta,\rho}(\cdot|\rvx_{\leq t-1}), \DCal(\cdot|\rvx_{\leq t-1})\Big).
\end{align}
Next, we focus on bounding the term ({\rm II}). As per notation, for any $\btheta \in \Theta$, we have
\begin{align}
\RCEW(\btheta) - \RCENW(\btheta) &= \frac{1}{N}\sum_{i=1}^N \left( \mathbb{E}\big[\ell^\omega(\rvx^{(i)}; \btheta)\big] - \ell^\omega(\rvx^{(i)}; \btheta)\right). 
\end{align}
Thus, for a fixed $\btheta \in \Theta$, we have
\begin{align}
\label{eq:delta_prob_bound}
&\mathbb{P}\left(\RCE({\btheta}) - \RCENW(\btheta) \geq \gamma \right) 
= \mathbb{P}\left(\frac{1}{N}\sum_{i=1}^N \left( \mathbb{E}\big[\ell^\omega(\rvx^{(i)}; \btheta)\big] - \ell^\omega(\rvx^{(i)}; \btheta)\right) \geq \gamma \right) \nonumber \\
&\qquad\qquad\qquad = \mathbb{P}\left(\sum_{i=1}^N \left( \mathbb{E}\big[\ell^\omega(\rvx^{(i)}; \btheta)\big] - \ell^\omega(\rvx^{(i)}; \btheta)\right)/C \geq N\cdot\frac{\gamma}{C} \right) \nonumber \\
&\qquad\qquad\qquad \overset{(i)}{\leq}  \exp\left(-\frac{N\gamma^2}{2(\sum_{t}V_t  + \frac{1}{3}C\gamma)}\right),
\end{align} 
where $(i)$ follows from \eqref{eq:KD_loss_concentration} with $\epsilon = \frac{\gamma}{C}$. With some algebra, one can see that the right hand side of \eqref{eq:delta_prob_bound} is bounded by $\delta / |\Theta|$ when
\begin{align}
\gamma \geq \frac{2C}{3N}\cdot\log\left({|\Theta|}/{\delta}\right) + \sqrt{\frac{2}{N}\cdot\sum_{t}V_t\cdot\log\left({|\Theta|}/{\delta}\right)}.  
\end{align}
(To see this, set the right hand side of \eqref{eq:delta_prob_bound} to $\delta/|\Theta|$ to get a quadratic of the form $\gamma^2 = a \gamma + b$ with $a, b \ge 0$ and note that its non-negative root is $\leq a + \sqrt{b}.$ All $\gamma \geq a + \sqrt{b}$ will make the right hand side of \eqref{eq:delta_prob_bound} $\leq \delta/|\Theta|.$)

Now, by taking union bound, with probability at least $1 - \delta$, for all $\btheta \in \Theta$, we have the following.
\begin{align}
\RCEW(\btheta) - \RCENW(\btheta) \leq \frac{2C}{3N}\cdot\log\left({|\Theta|}/{\delta}\right) + \sqrt{\frac{2}{N}\cdot\sum_{t}V_t\cdot\log\left({|\Theta|}/{\delta}\right)}.
\end{align}
Since the minimizer of the KD training objective $\thetaest$ is in $\Theta$, with probablity at least $1 - \delta$, we have
\begin{align}
\label{eq:delta_prob_bound_erm}
&({\rm II}) = \RCEW(\thetaest ) - \RCENW(\thetaest) \leq 
\frac{2C}{3N}\cdot\log\left({|\Theta|}/{\delta}\right) + \sqrt{\frac{2}{N}\cdot\sum_{t}V_t\cdot\log\left({|\Theta|}/{\delta}\right)}.
\end{align}
Now, the statement of Theorem~\ref{thm:token-level-generalization-bound} follows by combining \eqref{eq:decomp_erm}, \eqref{eq:div_bound_erm}, and \eqref{eq:delta_prob_bound_erm}.
\end{proof}


\section{KD can improve generalization via variance reduction}
\label{appen:kd-std-comparison}

Here, we leverage our novel generalization bounds to provide a theoretical justification for why KD can result in better generalization behavior compared to standard pre-training. In particular, we will focus on our bound in Theorem~\ref{thm:token-level-generalization-bound}.\footnote{One could draw similar conclusion from Theorem~\ref{thm:esr} by extending the arguments from \citet{menon21statistical} to the language modeling setting.} Note that, besides $|\Theta|$, $N$, and $T$ which are independent of the underlying training approach, there are three key quantities that dictate the generalization gap: (1) $\sum_{t}V_t$ which is related to the loss variance; (2) $C$ which is related to the extreme values that loss can take; and (3) the divergence between the teacher-provided next-token predictive distribution and the ground truth next-token distribution:
$$\textsc{Div}(\zeta, \omega) := \omega \cdot \sum_{t=1}^{T} \mathbb{E}\left[\mathsf{D}_{\rm TV}\left(P_{\bzeta,\rho}(\cdot|\rvx_{\leq t-1}), \DCal(\cdot|\rvx_{\leq t-1})\right)\right].$$
Note that, under Assumption~\ref{assm:loss_bound}, both KD and standard pre-training loss terms are bounded by $M$, allowing us to provide the same $C$ (as a function of $M$ and $T$) for \textit{both} KD and standard pre-training. Thus, we focus on the remaining two terms which relate to $\sum_{t}V_t$ and $\textsc{Div}(\zeta, \omega)$.

Note that standard pre-training, i.e., training without KD, corresponds to $\omega = 0$, which leads to  $\textsc{Div}(\zeta, \omega = 0) = 0$. In contrast, with $\omega > 0$, KD would incur a non-zero value for $\textsc{Div}(\zeta, \omega)$. On the other hand, as we will argue next, KD can lead to smaller value of the variance term $\sum_{t}V_t$. Thus, as long as the underlying teacher LM approximates the true next-token distribution well enough, it can lead to improved (student) performance or equivalently smaller generalization gap by striking a balance between the divergence (or bias) $\textsc{Div}(\zeta, \omega)$ and variance $\sum_{t}V_t$; as a result, realizing a form of \textit{bias vs. variance} trade-off for LM pre-training.

The variance reduction in the case of KD is the cleanest to observe by focusing on the last summand in $\sum_{t}V_t$, i.e., $V_T$. Towards this, recall from Assumption~\ref{assm:xi} that, for each $\btheta \in \Theta$, $V_T$ bounds the second-order moment of $\xi_T(\rvx; \btheta)$. 
Define the short-hand
\begin{equation}
    {P}^{(x_{t}, \omega)}_{\bzeta, \rho}(\cdot|\rvx_{\leq t-1)} := (1-\omega) \cdot \mathbbm{1}_{x_t} (\cdot) + \omega \cdot P_{\bzeta, \rho} (\cdot|\rvx_{\leq t-1)}
\end{equation}
and write
\begin{align}
\label{eq:lm-stability-analysis}
&\xi_T(\rvx; \btheta) = \mathbb{E}\left[\ell^\omega(\rvx; \btheta) | \rvx_{\leq T - 1} \right] - \mathbb{E}\left[\ell^\omega(\rvx; \btheta) | \rvx_{\leq T} \right] \nonumber \\
&= \mathbb{E}\left[\frac{1}{T}\sum_{t = 1}^{T} \mathsf{CE}\big( {P}^{(x_{t}, \omega)}_{\zeta, \rho}(\cdot | \rvx_{\leq t -1}),  P_{\btheta}(\cdot | \rvx_{\leq t -1})\big) \vert  \rvx_{\leq T - 1}  \right] \; - \; \nonumber \\
& \qquad \mathbb{E}\left[\frac{1}{T}\sum_{t = 1}^{T} \mathsf{CE}\big( {P}^{(x_{t}, \omega)}_{\zeta, \rho}(\cdot | \rvx_{\leq t -1}),  P_{\btheta}(\cdot | \rvx_{\leq t -1})\big) \vert  \rvx_{\leq T}  \right] \nonumber \\
&\overset{(i)}{=} \frac{1}{T}\sum_{t = 1}^{T-1} \mathsf{CE}\big( {P}^{(x_{t}, \omega)}_{\zeta, \rho}(\cdot | \rvx_{\leq t -1}),  P_{\btheta}(\cdot | \rvx_{\leq t -1})\big)  + \mathbb{E}\left[\frac{1}{T} \mathsf{CE}\big( {P}^{(x_{T}, \omega)}_{\zeta, \rho}(\cdot | \rvx_{\leq T -1}),  P_{\btheta}(\cdot | \rvx_{\leq T -1})\big) \vert  \rvx_{\leq T - 1}  \right] \nonumber \\
& \qquad - \frac{1}{T}\sum_{t = 1}^{T} \mathsf{CE}\big( {P}^{(x_{t}, \omega)}_{\zeta, \rho}(\cdot | \rvx_{\leq t -1}),  P_{\btheta}(\cdot | \rvx_{\leq t -1})\big) \nonumber \\
&\overset{(ii)}{=} \mathbb{E}\left[\frac{1}{T} \mathsf{CE}\big( {P}^{(x_{T}, \omega)}_{\zeta, \rho}(\cdot | \rvx_{\leq T -1}),  P_{\btheta}(\cdot | \rvx_{\leq T -1})\big) \vert  \rvx_{\leq T - 1}  \right] - \frac{1}{T} \mathsf{CE}\big( {P}^{(x_{T}, \omega)}_{\zeta, \rho}(\cdot | \rvx_{\leq T -1}),  P_{\btheta}(\cdot | \rvx_{\leq t -1})\big) \nonumber \\
&= (1 - \omega)\cdot \left(  \mathbb{E}\left[-\frac{1}{T} \cdot \log P_{\btheta}(x_{T} | \rvx_{\leq t -1})  \vert  \rvx_{\leq T - 1}  \right] +  \frac{1}{T} \cdot \log P_{\btheta}(x_{T} | \rvx_{\leq T -1}) \right)
\end{align}
where $(i)$ follows we can remove expectation for those terms that are functions of those random variables that we condition on; and $(ii)$ follows by removing the terms that cancel each other; and the last line follows as we have
$$
{P}^{(x_{T}, \omega)}_{\zeta, \rho}(\cdot | \rvx_{\leq T -1}) = (1-\omega)\cdot\mathds{1}_{x_{T}}(\cdot) + \omega\cdot P_{\bzeta, \rho}(\cdot | \rvx_{\leq T -1}).
$$
It follows from \eqref{eq:lm-stability-analysis} that, for any $\btheta \in \Theta$, we have 
\begin{align}
\label{eq:lm-stability-analysis2}
\mathbb{E}\left[\xi^2_T(\rvx; \btheta, \bzeta)\right] = (1 - \omega) \cdot {\rm Var}\left[ \frac{1}{T} \cdot \log P_{\btheta}(x_{T} | \rvx_{\leq t -1})~ \Big |~\rvx_{\leq T - 1} \right],
\end{align}
where ${\rm Var}\left[ \cdot | \cdot \right]$ denotes conditional variance. Note that \eqref{eq:lm-stability-analysis2} shows that $V_T$  decreases with $\omega$ in $[0,1]$. This highlights that KD, i.e., $\omega > 0$, would realize a smaller variance than standard pre-training, i.e., $\omega = 0$. Thus, to realize improved generalization via KD, one needs to select the distillation weight $\omega$ so that the \textit{variance reduction} via KD offsets the divergence term $\textsc{Div}(\zeta, \omega)$. In particular, when the teacher LM approximates the ground truth next-token distribution very well, i.e., $\textsc{Div}(\zeta, \omega)$ term is small even for a relatively large value of $\omega$, the variance reduction via KD becomes prominent, ensuring significant improvement over standard pre-training in terms of generalization performance.

\section{Bounding excess risk for KD}
\label{appen:excess_risk}

Different from the surrogate (empirical or population) risks utilized in the main text (cf.~Section~\ref{sec:theory}),
which utilize the cross-entropy loss as a surrogate loss, one could directly work with the risk defined with respect to a particular evaluation metric (and the corresponding loss) that one cares about. Since our training focuses on correct next-token prediction, we can focus on the accuracy of the next-token prediction under greedy-decoding as one such metric. This amounts to the following 
(population) risk with respect to \textit{$0/1$-loss}.
\begin{align}
R_{0/1}(\btheta) 
&:= \mathbb{E}_{\rvx \sim \DCal}\Big[\sum_{t=1}^{T}  \mathsf{1}\{ \argmax_{v} P_{\btheta}(v | \rvx_{\leq t-1}) \neq x_{t} \Big] \nonumber \\
&= \sum_{t=1}^{T}  \mathbb{E}_{\rvx_{\leq t-1} \sim \DCal} \Big[ \sum_{v \in \VC} \DCal(v|\rvx_{\leq t-1})\cdot \mathsf{1}\{ \argmax_{v'} P_{\btheta}(v'| \rvx_{\leq t-1}) \neq v \}\Big].
\label{eq:acc_risk}
\end{align}

A large body of literature~\citep[see, e.g.,][and references therein]{bartlett2006convexity, zhang2004statistical, steinwart2007compare, pires2016multiclass} has studied \textit{calibration functions} that enable converting bounds on \textit{excess surrogate risk} to control the \textit{excess risk}. Applying the calibration functions for the cross-entropy loss~\citep{pires2016multiclass}, we obtain the following bound on the excess risk for next-token prediction:
\begin{align}
    R_{0/1}(\thetaest) - R_{0/1}(\btheta^*) \leq \mathsf{g}^{-1}\left( R(\thetaest) - R(\btheta^{*}) \right),
\end{align}
where $\mathsf{g}^{-1}(\cdot)$ denotes the inverse of the function $\mathsf{g}:\epsilon \mapsto \frac{1}{2}\big((1 - \epsilon)\log (1 - \epsilon) + (1 + \epsilon)\log (1 + \epsilon) \big).$

\section{Experimental Setup Details}
\label{sec:exp-setup-details}
\textbf{Model architectures and pre-training data.}
We work with standard decoder-only Transformer-based LMs. Our small model (\slm{}) has 1.5B parameters. It comprises a 44 layer Transformer network with model dimension 1024, MLP hidden dimension 8192, and 4 attention heads. 
For the larger LM, we employ 2.8B parameter models consisting of 92 layer Transformer networks with model dimension 1024, MLP hidden dimension 8192, and 4 attention heads based on multi-query attention~\citep{shazeer2019mqa}.
We use a SentencePiece tokenizer~\citep{kudo2018sentencepiece} from \cite{du2022glam} with a vocabulary size of 256$K$.

We pre-train all LMs on the Pile dataset~\citep{gao2020thepile} by minimizing the {UL2 objective}~\citep{tay2023ul2}
with a mixture of four tasks: (1) \textit{causal LM} task; (2) \textit{prefix LM} task with mean prefix length of $1/4$th the sequence length,
(3) \textit{span corruption task} with $r=15\%$ of the tokens corrupted and mean corrupted span length $\mu=3$; and
(4) \textit{span corruption task} with $r=50\%$ of the tokens corrupted and mean corrupted span length $\mu=32$. The four tasks are mixed at a ratio of $6$:$2$:$1$:$1$.
We pre-train LMs for approximately 545 billion tokens, with a batch size of 2048 and input sequence length of 1280. This translates to a little over two epochs on the Pile data. As for the optimization method, we utilize Adafactor algorithm \citep{shazeer2018adafactor}. 
We use a cosine learning rate decay schedule with a peak learning rate of $0.001$, $4000$ warmup steps and final learning rate of $0.0001.$ Training is done on 1024 TPU-v5e chips with JAX~\citep{jax2018github}
and SeqIO~\citep{roberts2022t5x}.

\section{Few-shot evaluation tasks}
\label{appen:few-shot-benchmarks}

We performed a comprehensive few-shot evaluation of pre-trained LMs on $28$ benchmarks. Below, we list these by organizing them according to the corresponding domain.

\textbf{World Knowledge}:
NQ-Open \citep{lee2019latent_nq_open},
TriviaQA \citep{joshi2017triviaqa}, 
TyDiQA-NoContext (English)\citep{clark2020tydiqa},
Web Questions \citep{berant2013semantic}.

\textbf{Reading Comprehension:}
RACE-M, RACE-H \citep{lai2017race},
SQuADv2 \citep{squad2},
TyDiQA-GoldP (English)\citep{clark2020tydiqa}.

\textbf{Commonsense Reasoning:}
ARC (Easy) and ARC (Challenge) \citep{allenai:arc},
HellaSwag \citep{zellers2019hellaswag},
OpenBookQA \citep{OpenBookQA2018},
PiQA \citep{PIQA2020},
StoryCloze \citep{storycloze2016},
Winogrande \citep{Winogrande2020}.

\textbf{SuperGLUE~\citep{superglue2019}:}
BoolQ \citep{boolq2019},
CB \citep{cb2019},
COPA \citep{gordon2012copa},
RTE \citep{rte2006},
WiC \citep{wic2018},
WSC \citep{winograd2012},
MultiRC \citep{multirc2018},
ReCoRD \citep{record2018}.

\textbf{Natural Language Generation (NLG):}
English portions of the three benchmarks -- XLSum~\citep{hasan2021xlsum}, XSum~\citep{narayan2018dont_xsum} and WikiLingua~\citep{ladhak2020wikilingua}.

\textbf{Open-ended Cloze task:}~LAMBADA~\cite{paperno2016lambada}.

\textbf{Code generation:} Mostly Basic Python Problems (MBPP) ~\citep{austin2021program}.

\section{Additional Few-shot evaluation results}
\label{app:more_experiments}

Table~\ref{tab:evals_grouped_full} is an expansion of Table~\ref{tab:evals_all_grouped} in the main text.
All evaluations are $1$-shot, except for MBPP which is $3$-shot.
In the metric column, \textit{EM}, \textit{Acc}, and \textit{Rg2} are abbreviations for \textit{Exact Match},  \textit{Accuracy}, and \textit{Rouge2}, respectively.
For MBPP, the metric is the fraction of success ignoring challenge problems. As mentioned in Section~\ref{sec:exp-setup}, for each benchmark, we typically report the corresponding prevalent metric in the literature. For TyDiQA benchmarks, we report the F1 score as opposed to \textit{EM} as it is the primary metric in \citet{clark2020tydiqa}. For MultiRC in SuperGLUE, we report F1 metric as per \citet{du2022glam}.

\begin{table}[ht]
  \caption{
  \textbf{Comprehensive few-shot performance of pre-trained LMs.}~\slm{} serves as the teacher LM for \salt{} \& \saltds{} during the KD phase of their pre-training and for \rkd{} throughout its pre-training. \baseline{} employs standard pre-training \textit{without} KD from \slm{}. \salt{} and \saltds{} already outperform \baseline in terms of average few-shot performance at $70$\% of their training step budget, thereby improving both training efficiency and model quality. \rkd{}, i.e., naively preforming KD from the small model through the pre-training, performs much worse than \baseline{}. The best and second-best results for each domain are \textbf{boldfaced} and \underline{underlined}, respectively.
  }
  \label{tab:evals_grouped_full}
  \centering
   \setlength{\tabcolsep}{1pt}
   \scalebox{0.79}{
   \renewcommand{\arraystretch}{1.2}
    \begin{tabular}{ccccccccccccccc}
    \toprule
    \textbf{Domain} & \textbf{Dataset} & \textbf{Metric} && \slm && \baseline && \rkd &&
    \multicolumn{2}{c}{\salt} &&
    \multicolumn{2}{c}{\saltds} \\
    \cmidrule{7-7} \cmidrule{9-9}  \cmidrule{11-12} \cmidrule{14-15}
    & &  &&  && \textit{@100\%} && \textit{@100\%} &&
    \textit{@70\%} & \textit{@100\%}&&
    \textit{@70\%}  & \textit{@100\%} \\
    & &  &&  &&  \textit{steps} && \textit{steps} && 
    \textit{steps} & \textit{steps} &&
    \textit{steps}  & \textit{steps} \\
    \midrule
\multirow{5}{*}{\rotatebox[origin=c]{90}{\parbox[c]{2cm}{\centering \bf World \\ Knowledge}}}
& NaturalQuestions-Open & \textit{EM} && 5.90 && 8.70 && 6.70 && \underline{9.40} & \textbf{10.10} && 8.40 & 9.00 \\
& TriviaQA & \textit{EM} && 30.09 && \underline{43.15} && 34.87 && 39.87 & \textbf{43.71} && 39.37 & 41.27 \\
& TyDiQA-NoContext & \textit{F1} && 22.20 && \textbf{28.20} && 26.10 && \underline{27.90} & 27.10 && 25.90 & 27.20 \\
& WebQuestions & \textit{EM} && 5.40 && 8.70 && 7.10 && 9.20 & \textbf{9.90} && 8.90 & \underline{9.40} \\
& \bf Domain average & && 15.90 && \underline{22.19} && 18.69 && 21.59 & \textbf{22.70} && 20.64 & 21.72 \\ [1.5mm]
\midrule
\multirow{5}{*}{\rotatebox[origin=c]{90}{\parbox[c]{2.5cm}{\centering \bf Reading \\ Comprehension}}} & RACE-M & \textit{Acc} && 52.60 && 57.00 && 54.00 && \underline{58.60} & \textbf{58.90} && 57.90 & 58.40 \\
& RACE-H & \textit{Acc} &&37.50 && \textbf{42.30} && 39.70 && 42.20 & \textbf{42.30} && 42.10 & \textbf{42.30} \\
& SQuADv2 & \textit{EM} && 43.30 && 54.80 && 50.90 && 54.60 & 55.90 && \underline{57.60} & \textbf{57.90} \\
& TyDiQA-GoldP & \textit{F1} && 51.80 && 57.90 && 59.40 && 58.80 & \textbf{61.10} && 59.80 & \textbf{61.10} \\
& \bf Domain average & && 46.30 && 53.00 && 51.00 && 53.55 & \underline{54.55} && 54.35 & \textbf{54.93} \\[1.5mm]
\midrule
\multirow{8}{*}{\rotatebox[origin=c]{90}{\parbox[c]{2.5cm}{\centering \bf Commonsense \\ Reasoning}}} & ARC-E & \textit{Acc} && 64.60 && 68.40 && 66.00 && 67.60 & 67.60 && \textbf{69.40} & \underline{69.00} \\
& ARC-C & \textit{Acc} && 32.40 && 37.10 && 33.70 && 38.00 & \textbf{38.40} && \underline{38.10} & 37.30 \\
& HellaSwag & \textit{Acc} && 56.00 && 62.80 && 56.20 && 62.00 & \underline{63.30} && 63.10 & \textbf{63.80} \\
& OpenBookQA & \textit{Acc} && 48.00 && \textbf{50.00} && 45.80 && 47.20 & \underline{48.20} && 47.60 & \underline{48.20} \\
& PiQA & \textit{Acc} && 72.00  && \textbf{75.40} && 72.60 && 73.20 & 73.70 && \underline{74.10} & 73.90 \\
& StoryCloze & \textit{Acc} && 73.10 && \textbf{77.20} && 73.70 && 76.90 & 76.80 && 77.00 & \underline{77.10} \\
& WinoGrande & \textit{Acc} && 58.20 && 63.00 && 60.10 && 64.00 & 63.70 && \underline{64.70} & \textbf{65.40} \\
& \bf Domain average & && 57.76 && 61.99 && 58.30 && 61.27 & 61.67 && \underline{62.00} & \textbf{62.10} \\
\midrule
& LAMBADA & \textit{Acc} && 26.90 && 36.20 && 31.10 && \underline{50.70} & 48.30 && 48.00 & \textbf{53.00} \\ [1.5mm]
\midrule
\multirow{9}{*}{\rotatebox[origin=c]{90}{\parbox[c]{2.5cm}{\centering \bf SuperGLUE}}}& BoolQ & \textit{Acc} && 63.40 && \underline{64.30} && 62.50 && 64.10 & 62.30 && \textbf{65.50} & \underline{64.30} \\
& CB & \textit{Acc} && 37.50 && \underline{58.90} && 50.00 && \textbf{60.70} & 53.60 && 55.40 & 53.60 \\
& COPA & \textit{Acc} && 77.00 && \underline{79.00} && 71.00 && 76.00 & 77.00 && \textbf{81.00} & 77.00 \\
& MultiRC & \textit{F1} && 53.80 && 54.20 && 53.50 && \underline{57.50} & \textbf{58.60} && 50.70 & 53.00 \\
& RTE & \textit{Acc} && 55.20 && 55.60 && \bf 59.90 && 57.80 & \underline{58.50} && 54.20 & \underline{58.50} \\
& ReCoRD & \textit{Acc} && 84.80 && 87.10 && 85.20 && 86.60 & 86.90 && \underline{87.20} & \textbf{87.30} \\
& WiC & \textit{Acc} && 48.40 && 47.20 && 47.20 && 49.80 & 48.10 && \underline{50.00} & \textbf{50.90} \\
& WSC & \textit{Acc} && 72.60 && 77.90 && 74.00 && 77.90 & 77.20 && \textbf{83.90} & \underline{80.00} \\
& \bf Domain average & && 61.59 && 65.53 && 62.91 && \textbf{66.30} & 65.28 && \underline{65.99} & 65.58 \\ [1.5mm]
\midrule
\multirow{4}{*}{\rotatebox[origin=c]{90}{\parbox[c]{1cm}{\centering \bf NLG}}} & G\textit{EM}-XLSum & \textit{Rg2} && 2.80 && 4.10 && 3.40 && 4.40 & 4.40 && \textbf{4.60} & \textbf{4.60} \\
& G\textit{EM}-XSum & \textit{Rg2} && 2.80 && 5.10 && 3.20 && 5.00 & 5.10 && \textbf{5.40} & \textbf{5.40} \\
& WikiLingua & \textit{Rg2} && 3.80 && \underline{4.60} && 3.60 && 4.50 & \textbf{4.70} && 4.40 & 4.50 \\
& \bf Domain average & && 3.13 && 4.60 && 3.40 && 4.63 & 4.73 && \underline{4.80} & \textbf{4.83} \\[1.5mm]
\midrule
& MBPP & \textit{Acc} && 9.60 && 16.20 && 11.40 && 15.60 & \underline{17.00} && 16.60 & \textbf{17.80} \\[1.5mm]
\midrule
& \bf  Average (28 tasks) & && 42.56 && 47.32 && 44.39 && 47.86 & \underline{47.94} && 47.89 & \textbf{48.26} \\
\bottomrule
\end{tabular}}
\end{table}

\clearpage

\section{Ablation study of various design choices in \salt{}}
\label{sec:exp-ablation}

In this section, we explore 
how various design choices pertaining \salt{} affect its final performance.

\noindent\textbf{Distillation from a better quality small model.}~So far we assumed that \slm{} is also pre-trained for the same number of tokens as the LLM. Since training for \slm{} is relatively cheaper, one could consider a scenario where one invests more compute resources in improving the small model if it can eventually be beneficial in improving the LLM quality via \salt{}. Towards this, we employ a small LM that is trained $\sim 2.5$ times longer -- $498$K steps vs. $208$K steps in Section~\ref{sec:exp-results}.\footnote{This approach aligns with the recent studies~\citep{touvron2023_llama,Gadre:2024} that train small LMs well beyond the their optimal compute budget as predicted by neural scaling laws~\citep{Hoffmann:2022}.} As evident in 
Table~\ref{tab:better_teacher_full},
\salt{} is indeed able to utilize the better small model as a teacher in the KD phase to further improve the LLM quality, as measured by the average few-shot performance.

\begin{table}
  \caption{\textbf{Effect of improved}~\slm{}\textbf{(comprehensive few-shot evaluation).}~\salt{} with a better teacher -- a \slm{} trained for 498K steps as opposed to 208K steps -- yields LLM with better average few-shot performance. For each benchmark, the best and second best results are \textbf{boldfaced} and \underline{underlined}, respectively.}
  \label{tab:better_teacher_full}
  \centering
   \setlength{\tabcolsep}{1pt}
   \scalebox{0.86}{
   \renewcommand{\arraystretch}{1.2}
    \begin{tabular}{ccccccc}
    \toprule
\textbf{Domain} & \textbf{Dataset} & \textbf{Metric} & \parbox[c]{2cm}{\centering \slm{} trained for \colorbox{yellow!20}{208K steps}} & \parbox[c]{2cm}{\centering \slm{} trained for \colorbox{cyan!20}{498K steps}} & \parbox[c]{2.8cm}{\centering \salt{} w/ KD from \slm{} trained for \colorbox{yellow!20}{208K steps}}& \parbox[c]{2.8cm}{\centering \salt{} w/ KD from \slm{} trained for \colorbox{cyan!20}{498K steps}}\\
\midrule
\multirow{5}{*}{\rotatebox[origin=c]{90}{\parbox[c]{2cm}{\centering \bf World \\ Knowledge}}}
  & NaturalQuestions-Open & \textit{EM} & 5.90 & 6.30 & \textbf{10.10} & \underline{9.00}\\
  & TriviaQA & \textit{EM} & 30.09 & 31.74 & \textbf{43.71} & \underline{41.61}\\
  & TyDiQA-NoContext & \textit{F1} & 22.20 & 23.80 & \textbf{27.10} & \underline{26.20}\\
  & WebQuestions & \textit{EM} & 5.40 & 7.60 & \textbf{9.90} & \underline{9.10}\\
  & \textbf{Domain average} &   & 15.90 & 17.36 & \textbf{22.70} & \underline{21.48}\\
\midrule
\multirow{5}{*}{\rotatebox[origin=c]{90}{\parbox[c]{2.5cm}{\centering \bf Reading \\ Comprehension}}}
  & RACE-M & \textit{Acc} & 52.60 & 54.40 & \textbf{58.90} & \underline{57.00}\\
  & RACE-H & \textit{Acc} & 37.50 & 39.40 & \textbf{42.30} & \underline{42.00}\\
  & SQuADv2 & \textit{EM} & 43.30 & 49.00 & \underline{55.90} & \textbf{57.90}\\
  & TyDiQA-GoldP & \textit{F1} & 51.80 & 55.90 & \textbf{61.10} & \underline{56.80}\\
  & \textbf{Domain average} &   & 46.30 & 49.67 & \textbf{54.55} & \underline{53.43}\\
\midrule
\multirow{8}{*}{\rotatebox[origin=c]{90}{\parbox[c]{2.5cm}{\centering \bf Commonsense \\ Reasoning}}}
  & ARC-E & \textit{Acc} & 64.60 & 65.50 & \underline{67.60} & \textbf{69.30}\\
  & ARC-C & \textit{Acc} & 32.40 & 34.30 & \underline{38.40} & \textbf{39.10}\\
  & HellaSwag & \textit{Acc} & 56.00 & 57.80 & \textbf{63.30} & \underline{63.20}\\
  & OpenBookQA & \textit{Acc} & 48.00 & 46.40 & \underline{48.20} & \textbf{49.00}\\
  & PiQA & \textit{Acc} & 72.00 & 72.90 & \underline{73.70} & \textbf{74.60}\\
  & StoryCloze & \textit{Acc} & 73.10 & 75.00 & \underline{76.80} & \textbf{76.90}\\
  & WinoGrande & \textit{Acc} & 58.20 & 59.40 & \underline{63.70} & \textbf{63.80}\\
  & \textbf{Domain average} &   & 57.76 & 58.76 & \underline{61.67} & \textbf{62.27}\\
\midrule
  & LAMBADA & \textit{Acc} & 26.90 & 37.80 & \textbf{48.30} & \underline{47.80}\\
\midrule
\multirow{9}{*}{\rotatebox[origin=c]{90}{\parbox[c]{2.5cm}{\centering \bf SuperGLUE}}}
  & BoolQ & \textit{Acc} & \underline{63.40} & 61.40 & 62.30 & \textbf{65.80}\\
  & CB & \textit{Acc} & 37.50 & 42.90 & \underline{53.60} & \textbf{73.20}\\
  & COPA & \textit{Acc} & 77.00 & \underline{78.00} & 77.00 & \textbf{79.00}\\
  & MultiRC & \textit{F1} & \underline{53.80} & 48.40 & \textbf{58.60} & 53.20\\
  & RTE & \textit{Acc} & 55.20 & 52.30 & \underline{58.50} & \textbf{61.70}\\
  & ReCoRD & \textit{Acc} & 84.80 & 85.50 & \underline{86.90} & \textbf{87.10}\\
  & WIC & \textit{Acc} & \underline{48.40} & 47.30 & 48.10 & \textbf{49.20}\\
  & WSC & \textit{Acc} & 72.60 & 72.30 & \underline{77.20} & \textbf{79.30}\\
  & \textbf{Domain average} &   & 61.59 & 61.01 & \underline{65.28} & \textbf{68.56}\\
\midrule
\multirow{4}{*}{\rotatebox[origin=c]{90}{\parbox[c]{1cm}{\centering \bf NLG}}}
  & \textit{GEM}-XLSum & \textit{Rg2} & 2.80 & 3.50 & \textbf{4.40} & \underline{4.30}\\
  & \textit{GEM}-XSum & \textit{Rg2} & 2.80 & 3.10 & \underline{5.10} & \textbf{5.60}\\
  & WikiLingua & \textit{Rg2} & 3.80 & 3.80 & \textbf{4.70} & \underline{4.40}\\
  & \textbf{Domain average} &   & 3.13 & 3.47 & \underline{4.73} & \textbf{4.77}\\
\midrule
  & MBPP & \textit{Acc} & 9.60 & 12.80 & \underline{17.00} & \textbf{17.40}\\
\midrule
  & \textbf{Average (28 tasks)} &   & 42.56 & 43.88 & \underline{47.94} & \textbf{48.70}\\
\bottomrule
\end{tabular}}
\end{table}

\noindent\textbf{Varying transition point.}~A key design choice for \salt{} is the selection of the transition point $n_{\rm KD}$ from KD phase (first stage) to standard training (second stage). Table~\ref{tab:loss_switching_point_ablation_full} shows 
few-shot performance of \salt{} as we vary the transition point. Note that \salt{} ensures quality gains for LLM with a wide range of values for $n_{\rm KD}$ while demonstrating an inverted U-shape for LLM quality. We see consistent performance improvement from $n_{\rm KD} = 0$ (equivalent to \baseline{}) to $n_{\rm KD} = 60$K which eventually degrades at $n_{\rm KD}=208$K (equivalent to \rkd{}). Given the training overhead of KD phase (see discussion in Section~\ref{sec:exp-results}), smaller value of $n_{\rm KD}$ helps ensure training efficiency gains via \salt{}. Thus, we worked with $n_{\rm KD}=36$K in Section~\ref{sec:exp-results} as $n_{\rm KD}=60$K only provides marginal quality gains if one takes into account the increased training cost due to longer KD phase.

\begin{table}[t]
   \caption{\textbf{Effect of varying transitions step (comprehensive few-shot evaluation).}~The performance improvement via \salt{} over \baseline{} is stable in a wide range of $n_{\rm KD}$ (20k to 60k steps). Eventually, with much larger $n_{\rm KD}$, \salt{} performance degrades significantly (208k steps). For each benchmark, the best and second best results are \textbf{boldfaced} and \underline{underlined}, respectively.}
  \label{tab:loss_switching_point_ablation_full}
  \centering
   \setlength{\tabcolsep}{2pt}
   \scalebox{0.79}{
   \renewcommand{\arraystretch}{1.25}
    \begin{tabular}{ccccccccc}
    \toprule
    \textbf{Domain} & \textbf{Dataset} & \textbf{Metric} & \slm & \baseline &   \parbox[c]{1.8cm}{\centering \salt{} w/ $n_{\rm KD} = 20$K} &   \parbox[c]{1.8cm}{\centering \salt{} w/ $n_{\rm KD} = 36$K} &   \parbox[c]{1.8cm}{\centering \salt{} w/ $n_{\rm KD} = 60$K} &  \parbox[c]{2.0cm}{\centering \salt{} w/ $n_{\rm KD} = 208$K (\rkd)} \\
    \midrule
\multirow{5}{*}{\rotatebox[origin=c]{90}{\parbox[c]{2cm}{\centering \bf World \\ Knowledge}}}
  & NaturalQuestions-Open & \textit{EM} & 5.90 & 8.70 & 8.90 & \textbf{10.10} & \underline{9.30} & 6.70\\
  & TriviaQA & \textit{EM} & 30.09 & \underline{43.15} & 41.52 & \textbf{43.71} & 42.84 & 34.87\\
  & TyDiQA-NoContext & \textit{F1} & 22.20 & \textbf{28.20} & 26.40 & \underline{27.10} & 26.60 & 26.10\\
  & WebQuestions & \textit{EM} & 5.40 & \underline{8.70} & 8.20 & \textbf{9.90} & 8.60 & 7.10\\
  & \textbf{Domain average} &   & 15.90 & \underline{22.19} & 21.26 & \textbf{22.70} & 21.83 & 18.69\\
\midrule
\multirow{5}{*}{\rotatebox[origin=c]{90}{\parbox[c]{2.5cm}{\centering \bf Reading \\ Comprehension}}}
  & RACE-M & \textit{Acc} & 52.60 & 57.00 & \underline{58.70} & \textbf{58.90} & 58.60 & 54.00\\
  & RACE-H & \textit{Acc} & 37.50 & \textbf{42.30} & 41.00 & \textbf{42.30} & 42.10 & 39.70\\
  & SQuADv2 & \textit{EM} & 43.30 & 54.80 & 55.30 & \textbf{55.90} & \underline{55.50} & 50.90\\
  & TyDiQA-GoldP & \textit{F1} & 51.80 & 57.90 & 56.50 & \textbf{61.10} & 59.30 & \underline{59.40}\\
  & \textbf{Domain average} &   & 46.30 & 53.00 & 52.88 & \textbf{54.55} & \underline{53.88} & 51.00\\
\midrule
\multirow{8}{*}{\rotatebox[origin=c]{90}{\parbox[c]{2.5cm}{\centering \bf Commonsense \\ Reasoning}}}
  & ARC-E & \textit{Acc} & 64.60 & \textbf{68.40} & 67.80 & 67.60 & \textbf{68.40} & 66.00\\
  & ARC-C & \textit{Acc} & 32.40 & 37.10 & 38.10 & \underline{38.40} & \textbf{38.70} & 33.70\\
  & HellaSwag & \textit{Acc} & 56.00 & 62.80 & 62.80 & \textbf{63.30} & \underline{62.90} & 56.20\\
  & OpenBookQA & \textit{Acc} & 48.00 & \textbf{50.00} & 48.00 & \underline{48.20} & \underline{48.20} & 45.80\\
  & PiQA & \textit{Acc} & 72.00 & \textbf{75.40} & \textbf{75.40} & 73.70 & 74.40 & 72.60\\
  & StoryCloze & \textit{Acc} & 73.10 & \textbf{77.20} & \underline{76.90} & 76.80 & 76.50 & 73.70\\
  & WinoGrande & \textit{Acc} & 58.20 & 63.00 & \underline{63.40} & \textbf{63.70} & 62.00 & 60.10\\
  & \textbf{Domain average} &   & 57.76 & \textbf{61.99} & \underline{61.77} & 61.67 & 61.59 & 58.30\\
\midrule
  & LAMBADA & \textit{Acc} & 26.90 & 36.20 & 44.70 & \underline{48.30} & \textbf{53.30} & 31.10\\
\midrule
\multirow{9}{*}{\rotatebox[origin=c]{90}{\parbox[c]{2.5cm}{\centering \bf SuperGLUE}}}
  & BoolQ & \textit{Acc} & 63.40 & \textbf{64.30} & \underline{63.90} & 62.30 & 63.80 & 62.50\\
  & CB & \textit{Acc} & 37.50 & \underline{58.90} & \textbf{60.70} & 53.60 & 55.40 & 50.00\\
  & COPA & \textit{Acc} & \underline{77.00} & \textbf{79.00} & 76.00 & \underline{77.00} & \underline{77.00} & 71.00\\
  & MultiRC & \textit{F1} & 53.80 & 54.20 & 53.80 & \textbf{58.60} & \underline{55.20} & 53.50\\
  & RTE & \textit{Acc} & 55.20 & 55.60 & 52.30 & 58.50 & \textbf{59.90} & \textbf{59.90}\\
  & ReCoRD & \textit{Acc} & 84.80 & \textbf{87.10} & \underline{86.90} & \underline{86.90} & 86.70 & 85.20\\
  & WIC & \textit{Acc} & 48.40 & 47.20 & \textbf{51.30} & 48.10 & \underline{50.00} & 47.20\\
  & WSC & \textit{Acc} & 72.60 & \textbf{77.90} & 77.50 & 77.20 & \textbf{77.90} & 74.00\\
  & \textbf{Domain average} &   & 61.59 & \underline{65.53} & 65.30 & 65.28 & \textbf{65.74} & 62.91\\
\midrule
\multirow{4}{*}{\rotatebox[origin=c]{90}{\parbox[c]{1cm}{\centering \bf NLG}}}
  & \textit{GEM}-XLSum & \textit{Rg2} & 2.80 & 4.10 & \underline{4.50} & 4.40 & \textbf{4.70} & 3.40\\
  & \textit{GEM}-XSum & \textit{Rg2} & 2.80 & \underline{5.10} & \textbf{5.80} & \underline{5.10} & 4.80 & 3.20\\
  & WikiLingua & \textit{Rg2} & 3.80 & \underline{4.60} & 4.30 & \textbf{4.70} & \underline{4.60} & 3.60\\
  & \textbf{Domain average} &   & 3.13 & 4.60 & \textbf{4.87} & \underline{4.73} & 4.70 & 3.40\\
\midrule
  & MBPP & \textit{Acc} & 9.60 & 16.20 & \underline{16.60} & \textbf{17.00} & 16.40 & 11.40\\
\midrule
  & \textbf{Average (28 tasks)} &   & 42.56 & 47.32 & 47.40 & \underline{47.94} & \textbf{47.99} & 44.39 \\
\bottomrule
\end{tabular}}
\end{table}

\noindent\textbf{Different transition strategies.}~In our study thus far, we have worked with \textit{Step} transition between the two training stages in \salt{} where we abruptly stop performing KD after $n_{\rm KD}$ training steps. Looking at Figure~\ref{fig:training-curves}, this causes an abrupt change in the model behavior during training, as observed in the next-token prediction accuracy curve for the training set (see similar behavior for log-perplexity in Figure~\ref{fig:log-perplexity}). This raises a question if a smoother transition between the two stages can improve the training stability and thereby ensure higher final LLM quality. While there is a large space of potential choices of such smooth transition strategies, here we explore two natural candidates: (1) \textit{Linear decay} where we linearly decrease the distillation loss weight to $0$ between $n_{{\rm KD},1}=32$K and $n_{\rm KD}=36$K steps; and (2) \textit{Linear ratio decay} where we linearly decrease the ratio of distillation loss weight and standard loss weight $\frac{\omega}{1 - \omega}$ to $0$ between $n_{{\rm KD},1}=32$K and $n_{{\rm KD},2}=36$K training steps. As recorded in Table~\ref{tab:loss_switching_method_ablation_full}, the step transition constitutes a reasonable design choice for \salt{} as it outperforms both the considered  alternatives in terms of average few-shot performance of the resulting pre-trained LLM.

\begin{table}
  \caption{\textbf{Effect of different transition strategies (comprehensive few-shot evaluation.}~The Step transition used in this work (cf.~Algorithm~\ref{alg:stl-two-stage}) performs well compared to two natural alternative strategies. For each benchmark, the best and second best results are \textbf{boldfaced} and \underline{underlined}, respectively.}
  \label{tab:loss_switching_method_ablation_full}
  \centering
   \setlength{\tabcolsep}{2pt}
   \scalebox{0.86}{
   \renewcommand{\arraystretch}{1.2}
    \begin{tabular}{cccccccc}
    \toprule
\textbf{Domain} & \textbf{Dataset} & \textbf{Metric} & \slm{} & \baseline{} & \salt{} w/ & \salt{} w/ & \salt{} w/\\
    & & & & & 
    Step & Linear decay & Linear ratio decay\\
\midrule
\multirow{5}{*}{\rotatebox[origin=c]{90}{\parbox[c]{2cm}{\centering \bf World \\ Knowledge}}}
  & NaturalQuestions-Open & \textit{EM} & 5.90 & \underline{8.70} & \textbf{10.10} & 8.20 & 8.10\\
  & TriviaQA & \textit{EM} & 30.09 & 43.15 & \textbf{43.71} & 43.46 & \underline{43.51}\\
  & TyDiQA-NoContext & \textit{F1} & 22.20 & \underline{28.20} & 27.10 & \textbf{28.40} & 27.20\\
  & WebQuestions & \textit{EM} & 5.40 & \underline{8.70} & \textbf{9.90} & 8.20 & 8.40\\
  & \textbf{Domain average} &   & 15.90 & \underline{22.19} & \textbf{22.70} & 22.07 & 21.80\\
\midrule
\multirow{5}{*}{\rotatebox[origin=c]{90}{\parbox[c]{2.2cm}{\centering \bf Reading \\ Comprehension}}}
  & RACE-M & \textit{Acc} & 52.60 & 57.00 & \textbf{58.90} & \underline{57.90} & 57.40\\
  & RACE-H & \textit{Acc} & 37.50 & \underline{42.30} & \underline{42.30} & 42.10 & \textbf{43.50}\\
  & SQuADv2 & \textit{EM} & 43.30 & 54.80 & 55.90 & \underline{56.40} & \textbf{57.10}\\
  & TyDiQA-GoldP & \textit{F1} & 51.80 & 57.90 & \textbf{61.10} & \underline{58.30} & 57.80\\
  & \textbf{Domain average} &   & 46.30 & 53.00 & \textbf{54.55} & 53.68 & \underline{53.95}\\
\midrule
\multirow{8}{*}{\rotatebox[origin=c]{90}{\parbox[c]{2.5cm}{\centering \bf Commonsense \\ Reasoning}}}
  & ARC-E & \textit{Acc} & 64.60 & 68.40 & 67.60 & \underline{68.60} & \textbf{68.70}\\
  & ARC-C & \textit{Acc} & 32.40 & 37.10 & 38.40 & \underline{38.60} & \textbf{39.80}\\
  & HellaSwag & \textit{Acc} & 56.00 & 62.80 & \underline{63.30} & \underline{63.30} & \textbf{63.50}\\
  & OpenBookQA & \textit{Acc} & 48.00 & \textbf{50.00} & \underline{48.20} & 48.00 & 47.40\\
  & PiQA & \textit{Acc} & 72.00 & \textbf{75.40} & 73.70 & \underline{74.60} & 73.90\\
  & StoryCloze & \textit{Acc} & 73.10 & \textbf{77.20} & \underline{76.80} & 76.60 & 76.50\\
  & WinoGrande & \textit{Acc} & 58.20 & 63.00 & \textbf{63.70} & 62.70 & \underline{63.10}\\
  & \textbf{Domain average} &   & 57.76 & \textbf{61.99} & 61.67 & 61.77 & \underline{61.84}\\
\midrule
  & LAMBADA & \textit{Acc} & 26.90 & 36.20 & \textbf{48.30} & 40.50 & \underline{42.60}\\
\midrule
\multirow{9}{*}{\rotatebox[origin=c]{90}{\parbox[c]{2.5cm}{\centering \bf SuperGLUE}}}
  & BoolQ & \textit{Acc} & 63.40 & 64.30 & 62.30 & \textbf{67.90} & \underline{66.50}\\
  & CB & \textit{Acc} & 37.50 & \textbf{58.90} & \underline{53.60} & 44.60 & 46.40\\
  & COPA & \textit{Acc} & 77.00 & \underline{79.00} & 77.00 & \underline{79.00} & \textbf{81.00}\\
  & MultiRC & \textit{F1} & 53.80 & 54.20 & \underline{58.60} & 53.90 & \textbf{61.60}\\
  & RTE & \textit{Acc} & 55.20 & 55.60 & \textbf{58.50} & \underline{56.30} & 55.20\\
  & ReCoRD & \textit{Acc} & 84.80 & \underline{87.10} & 86.90 & 87.00 & \textbf{87.30}\\
  & WIC & \textit{Acc} & \underline{48.40} & 47.20 & 48.10 & 46.60 & \textbf{50.50}\\
  & WSC & \textit{Acc} & 72.60 & 77.90 & 77.20 & \underline{78.60} & \textbf{78.90}\\
  & \textbf{Domain average} &   & 61.59 & \underline{65.53} & 65.28 & 64.24 & \textbf{65.92}\\
\midrule
\multirow{4}{*}{\rotatebox[origin=c]{90}{\parbox[c]{1cm}{\centering \bf NLG}}}
  & \textit{GEM}-XLSum & \textit{Rg2} & 2.80 & 4.10 & 4.40 & \textbf{4.70} & \underline{4.50}\\
  & \textit{GEM}-XSum & \textit{Rg2} & 2.80 & \textbf{5.10} & \textbf{5.10} & 4.60 & \textbf{5.10}\\
  & WikiLingua & \textit{Rg2} & 3.80 & 4.60 & \underline{4.70} & \textbf{4.80} & 4.60\\
  & \textbf{Domain average} &   & 3.13 & 4.60 & \textbf{4.73} & 4.70 & \textbf{4.73}\\
\midrule
  & MBPP & \textit{Acc} & 9.60 & 16.20 & \textbf{17.00} & 15.20 & \textbf{17.00}\\
\midrule
  & \textbf{Average (28 tasks)} &   & 42.56 & 47.32 & \textbf{47.94} & 47.11 & \underline{47.75}\\
\bottomrule
\end{tabular}}
\end{table}

\clearpage

\section{Log perplexity of the models}
Figure~\ref{fig:log-perplexity} shows the log perplexity of the \salt{} and \rkd{} pre-trained models along with \baseline{} and \slm{}. The log perplexity for \rkd{} stays at a higher level than even \slm{}. Recall that \rkd{} optimizes a sum of two losses -- KD loss with weight $\omega=0.667$ and the standard one-hot training loss with weight $1-\omega$. As the training log perplexity plotted in Figure~\ref{fig:log-perplexity} is the same as the standard hot training loss, the methods which directly optimize for that alone (\baseline{}, \slm{} and in the second stage, \salt{}) have lower log perplexity on training set than \rkd{} which optimizes additionally for distillation loss.
\begin{figure}[htbp]
\centering
\includegraphics[width=0.4\linewidth]{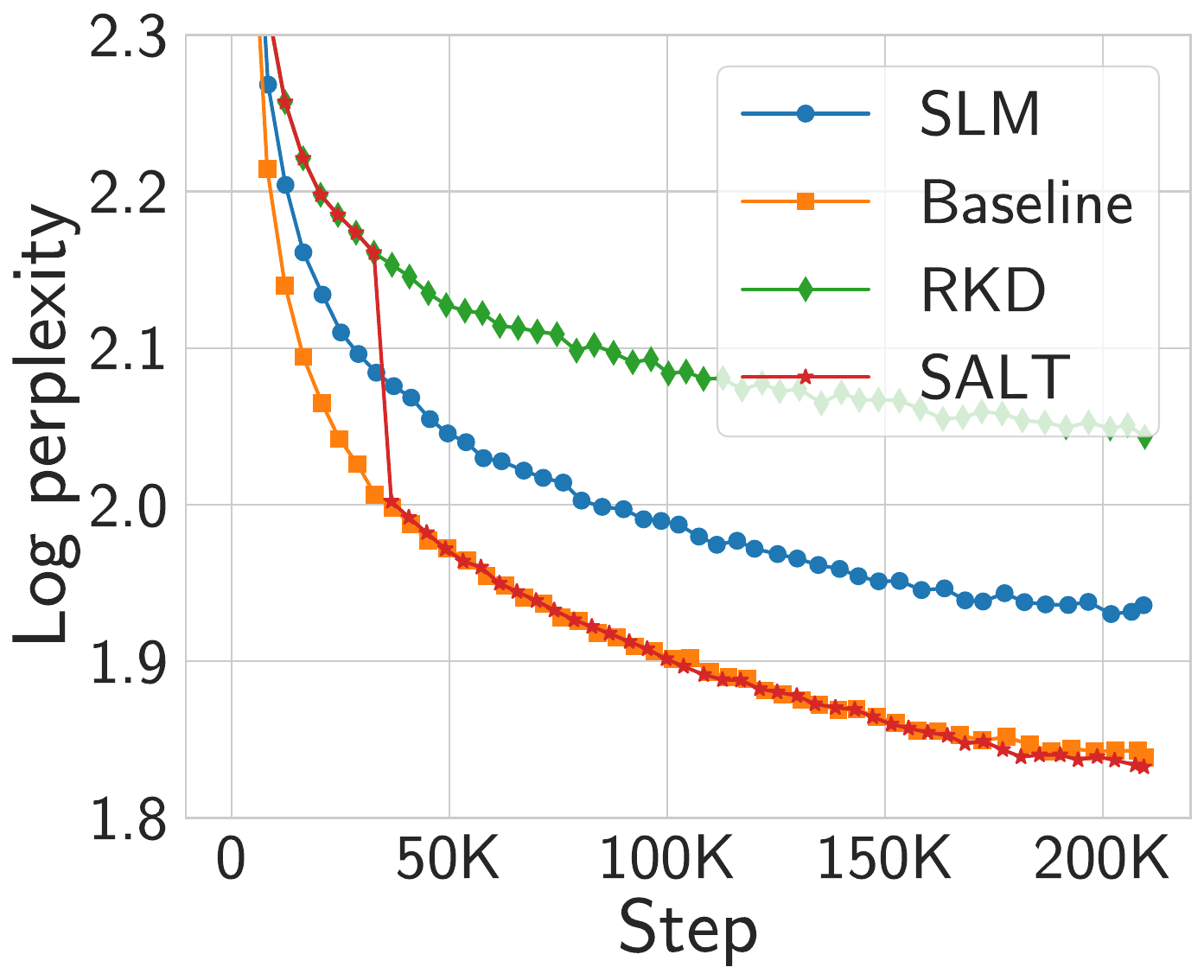}
\caption{Log perplexity for different models during their pre-training, as measured on a subset of the Pile training set.}
\label{fig:log-perplexity}
\end{figure}

\clearpage
\section{Additional results: Learning easy vs. hard instance via \salt{}}
\label{appen:difficulty_slice}

\noindent\textbf{Creation of different hardness buckets.}
For each evaluation benchmark, we first assign a relative rank to each test instance/example in the benchmark, representing its degree of difficulty. 
A test example with the lowest rank (easiest) is the one on which the small teacher LM achieves the largest task evaluation score, e.g., Rouge-2 metric for the XLSum task.
Similarly, subsequent test examples are assigned ranks in descending order of the task evaluation score achieved by the small teacher LM.
If two examples have the same evaluation score, the one with higher confidence score from the teacher (on its generated output) is deemed to have a lower rank.
Each test example is assigned to one of the three buckets: `easy`, `medium`, or `hard`, according to whether its difficulty rank is in the first, second, or third tertile, respectively.

In Tables \ref{tab:example_difficulty_squadv2}, \ref{tab:example_difficulty_triviaqa} and \ref{tab:example_difficulty-lambada}, we report the results for SQuAD-v2, TriviaQA and LAMBADA respectively, sliced by difficulty level.

\begin{table}[h]
  \caption{\textbf{Few-shot evaluation on different buckets of SQuAD-v2.}~Each number shows average Exact Match scores on the corresponding bucket. We use \hlg{gray}, \hla{green}, and \hlr{red} to highlight the results similar to, better than, and worse than \baseline{} performance, respectively.
  }
  \label{tab:example_difficulty_squadv2}
  \centering
\setlength{\tabcolsep}{2pt}
\scalebox{0.85}{
\renewcommand{\arraystretch}{0.9}
\begin{tabular}[t]{@{}lp{4cm}ccc@{}} 
\toprule  
&  \centering \bf Evaluation stage (steps) & \bf Easy & \bf Medium & \bf Hard \\
\midrule 
\slm & \centering Final (208K) & {1.00} & 0.30 & 0.00 \\ 
\midrule
\baseline{} & \centering \multirow{3}{*}{Early (36K)} & \hlg{0.86} & \hlg{0.41} & \hlg{0.23} \\ 
\rkd{} & & \hlg{0.86} & \hlr{0.37} & \hlr{0.17} \\
\salt{} & & \hlg{0.86} & \hlr{0.37} & \hlr{0.17}\\
\midrule
\baseline{} & \centering \multirow{3}{*}{Final (208K)} & \hlg{0.89} & \hlg{0.47} & \hlg{0.28} \\ 
\rkd{} & & \hla{0.91} & \hlr{0.42} & \hlr{0.20} \\ 
\salt{} & & \hlg{0.89} & \hla{0.50} & \hla{0.29} \\ 
\midrule 
\end{tabular}}
\end{table}

\begin{table}[h]
  \caption{\textbf{Few-shot evaluation on different buckets of TriviaQA.}~Each number shows average Exact Match scores on the corresponding bucket. We use \hlg{gray}, \hla{green}, and \hlr{red} to highlight the results similar to, better than, and worse than \baseline{} performance, respectively. 
  }
  \label{tab:example_difficulty_triviaqa}
  \centering
\setlength{\tabcolsep}{2pt}
\scalebox{0.85}{
\renewcommand{\arraystretch}{0.9}
\begin{tabular}[t]{@{}lp{4cm}ccc@{}} 
\toprule  
&  \centering \bf Evaluation stage (steps) & \bf Easy & \bf Medium & \bf Hard \\
\midrule 
\slm & \centering Final (208K) & {0.90} & 0.00 & 0.00 \\ 
\midrule
\baseline{} & \centering \multirow{3}{*}{Early (36K)} & \hlg{0.63} & \hlg{0.11} & \hlg{0.08} \\ 
\rkd{} & & \hla{0.67} & \hlrr{0.10} & \hlrr{0.06} \\
\salt{} & & \hla{0.67} & \hlrr{0.10} & \hlrr{0.06} \\
\midrule
\baseline{} & \centering \multirow{3}{*}{Final (208K)} & \hlg{0.80} & \hlg{0.28} & \hlg{0.22} \\ 
\rkd{} & & \hlrr{0.79} & \hlr{0.14} & \hlr{0.11} \\ 
\salt{} & & \hla{0.81} & \hlrr{0.27} & \hla{0.23} \\ 
\midrule 
\end{tabular}}

\end{table}

\begin{table}[htbp]
  \caption{\textbf{Few-shot evaluation on different buckets of LAMBADA.}~Each number shows average Accuracy on the corresponding bucket. We use \hlg{gray}, \hla{green}, and \hlr{red} to highlight the results similar to, better than, and worse than \baseline{} performance, respectively.  
  }
  \label{tab:example_difficulty-lambada}
  \centering
\setlength{\tabcolsep}{2pt}
\scalebox{0.85}{
\renewcommand{\arraystretch}{0.9}
\begin{tabular}[t]{@{}lp{4cm}ccc@{}} 
\toprule  
&  \centering \bf Evaluation stage (steps) & \bf Easy & \bf Medium & \bf Hard \\
\midrule 
\slm & \centering Final (208K) & {0.87} & 0.00 & 0.00 \\ 
\midrule
\baseline & \centering \multirow{3}{*}{Early (36K)} & \hlg{0.47} & \hlg{0.12} & \hlg{0.12} \\ 
\rkd & & \hla{0.56} & \hlrr{0.11} & \hlg{0.12} \\
\salt & & \hla{0.56} & \hlrr{0.11} & \hlg{0.12} \\
\midrule
\baseline & \centering \multirow{3}{*}{Final (208K)} & \hlg{0.70} & \hlg{0.29} & \hlg{0.28} \\ 
\rkd & &  \hlr{0.65} & \hlr{0.17} & \hlr{0.17} \\ 
\salt & & \hla{0.78} & \hla{0.38} & \hla{0.36} \\ 
\midrule 
\end{tabular}}
\end{table}

\end{document}